\documentclass{article}

\usepackage[T1]{fontenc}
\usepackage{iclr2027_conference,times}
\usepackage{amsmath,amssymb,amsthm}
\usepackage{graphicx}
\usepackage{booktabs}
\usepackage{array}
\usepackage{multirow}
\usepackage{subcaption}
\usepackage{placeins}
\usepackage{wrapfig}
\usepackage{microtype}
\usepackage{xcolor}
\usepackage{xspace}
\definecolor{CitationBlue}{HTML}{2457A7}
\usepackage[colorlinks=true,citecolor=CitationBlue,linkcolor=black,urlcolor=black,pdfborder={0 0 0}]{hyperref}
\usepackage{url}

\DeclareMathOperator{\Var}{Var}
\DeclareMathOperator{\clip}{clip}

\newcommand{\method}{\textsc{AlphaAnything}\xspace}

\newsavebox{\AlphaTableBox}
\newcommand{\AlphaTableHead}[1]{%
  \begingroup\renewcommand{\arraystretch}{1}%
  \begin{tabular}[c]{@{}c@{}}#1\end{tabular}%
  \endgroup
}
\newcommand{\AlphaFitTable}[1]{%
  \sbox{\AlphaTableBox}{#1}%
  \ifdim\wd\AlphaTableBox>\linewidth
    \resizebox{\linewidth}{!}{\usebox{\AlphaTableBox}}%
  \else
    \usebox{\AlphaTableBox}%
  \fi
}

\newtheorem{proposition}{Proposition}

\usepackage{tcolorbox}
\definecolor{AlphaAbstractBackground}{HTML}{F2F3F5}

\usepackage[letterpaper,left=2.5cm,right=2.5cm,top=2.25cm,bottom=2.5cm]{geometry}
\newlength{\AlphaFigureWidth}
\hypersetup{pdftitle={AlphaAnything: Extending Pretrained RGB Generators to RGBA},
  pdfauthor={Qijie Wang, Zhixin Xu, Zhe Li, Haowei Zhu, Honghao Qiao, Bo Peng, Dawei Li, Bin Wang}}

\newcommand{\AlphaResourceLinks}{%
  {\small\color{black}\renewcommand{\arraystretch}{1.45}%
    \begin{tabular}{@{}c@{\hspace{0.55em}}l@{}}
      \raisebox{-0.2em}{\includegraphics[height=1.2em]{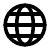}} & {\sffamily\bfseries Project Page:}\hspace{0.55em}%
      \href{https://wluli.github.io/AlphaAnything/}{\texttt{wluli.github.io/AlphaAnything/}}\\
      \raisebox{-0.2em}{\includegraphics[height=1.2em]{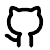}} & {\sffamily\bfseries Code:}\hspace{0.55em}%
      \href{https://github.com/WLuLi/AlphaAnything}{\texttt{github.com/WLuLi/AlphaAnything}}
    \end{tabular}%
  }%
}

\makeatletter
\renewcommand{\@maketitle}{%
  \vbox{\hsize\textwidth
    \raggedright
    \setlength{\parskip}{8pt}
    \includegraphics[width=\textwidth]{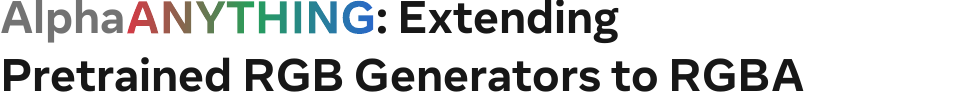}\par
    {\renewcommand{\arraystretch}{1.32}\setlength{\tabcolsep}{0pt}%
      \begin{tabular}{l}\@author\end{tabular}\par}
    \AlphaResourceLinks\par
  }%
}
\makeatother

\renewenvironment{abstract}{%
  \par\noindent
  \begin{tcolorbox}[width=\textwidth,colback=AlphaAbstractBackground,
    colframe=AlphaAbstractBackground,boxrule=0pt,arc=7pt,
    boxsep=0pt,left=12pt,right=12pt,top=12pt,bottom=12pt,
    before skip=8pt,after skip=12pt]
  \setlength{\parindent}{0pt}\setlength{\parskip}{0pt}
  \renewcommand{\method}{\mbox{\textsc{AlphaAnything}}\xspace}
  \noindent\includegraphics{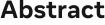}\par
  
  \fontsize{10}{12}\selectfont
}{%
  \par
  \end{tcolorbox}
}

\title{AlphaAnything: Extending Pretrained RGB Generators to RGBA}

\author{\normalsize\bfseries
Qijie Wang$^{1,2,}$\thanks{Equal contribution. Emails: \texttt{qijiewang02@gmail.com}, \texttt{xuzhixin.66@bytedance.com}.}\kern0.35em$^{,}$%
\setcounter{footnote}{2}\thanks{Work done during an internship at ByteDance.}
\quad Zhixin Xu$^{2,*}$ \quad Zhe Li$^{3,2}$ \quad Haowei Zhu$^{1}$ \\
\normalsize\bfseries Honghao Qiao$^{2}$ \quad Bo Peng$^{2}$ \quad Dawei Li$^{2,\dagger}$ \quad Bin Wang$^{1,\dagger}$ \\[4pt]
\normalfont $^{1}$Tsinghua University \quad $^{2}$ByteDance \quad $^{3}$Peking University \\
\normalfont $^{\dagger}$Corresponding authors}

\iclrfinalcopy

\begin{document}

\maketitle
\lhead{}
\renewcommand{\headrulewidth}{0pt}

\begin{abstract}
Existing RGBA generation methods commonly learn a transparency-specific representation through reconstruction, introducing an additional training stage and potentially altering the RGB decoding function inherited from pretrained models. We present \method, a lightweight framework that extends pretrained RGB generators to RGBA directly in their original latent space, without RGBA reconstruction pretraining. \method inserts LoRA modules into the generative transformer and attaches an alpha branch to the frozen RGB VAE decoder. Unidirectional RGB-to-alpha feature sharing allows opacity supervision to shape the generative latent while guaranteeing that, for any fixed latent, RGB decoding remains identical to that of the pretrained decoder. We further introduce Endpoint-Aware Alpha Loss (EA Loss), which derives opacity supervision from the complementary foreground--background weights used in compositing. It supports exact opacity endpoints and continuous partial transparency while emphasizing localized alpha errors. Experiments on text-to-RGBA and image-to-RGBA tasks across multiple pretrained RGB backbones demonstrate competitive generation quality, and controlled ablations validate the shared-latent training design and EA Loss.
\end{abstract}

\section{Introduction}
\label{sec:introduction}

Pretrained image generators can synthesize high-quality visual content from text or image conditions~\citep{rombach2022highresolution,wu2025qwenimage,wu2026visualgeneration}. Many downstream applications, however, require assets that can be directly layered and composited into new scenes, rather than RGB images tied to a fixed background~\citep{fontanella2024rgba,yin2025qwenlayered}. RGBA images support this use case by representing appearance and opacity separately: the RGB channels describe foreground color, while the alpha channel controls how the foreground is blended with its background~\citep{xu2017deepmatting,zhang2024layerdiffuse}. A practical RGBA extension should therefore add transparency generation while reusing the conditioning and appearance priors of a pretrained RGB model.

\begin{figure}[p]
\centering
\includegraphics[width=\AlphaFigureWidth,keepaspectratio]{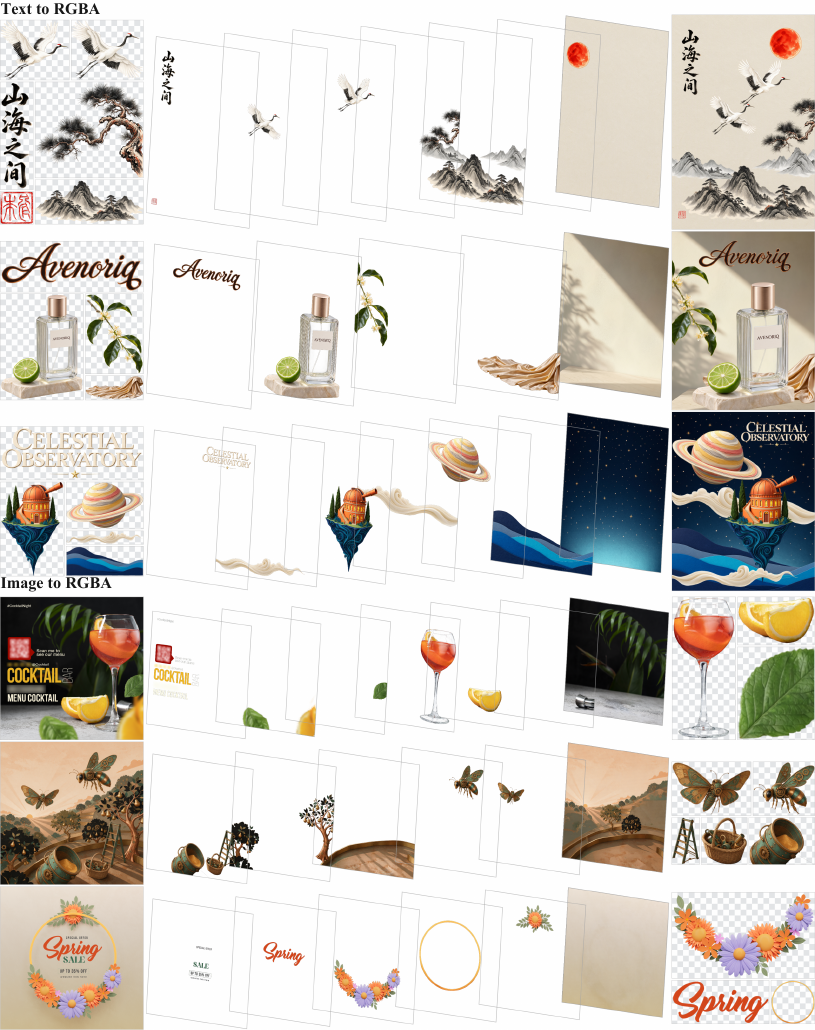}
\caption{Representative text- and image-to-RGBA results. Each prompt specifies an asset to generate or extract. For text-to-RGBA, independently generated layers are manually arranged into the displayed compositions.}

\label{fig:teaser}
\end{figure}

Existing RGBA generation methods commonly introduce a transparency-specific representation through reconstruction training. LayerDiffuse learns a transparency encoder and decoder with reconstruction supervision~\citep{zhang2024layerdiffuse}, while AlphaVAE and Qwen-Image-Layered train dedicated four-channel autoencoders~\citep{wang2025alphavae,yin2025qwenlayered}. Although effective, this strategy adds a separate representation-learning stage and couples subsequent generative training to the reconstructed RGBA latent space. Updating an RGB autoencoder for RGBA reconstruction can also change its pretrained RGB decoding function, with potential losses in color or high-frequency detail. Zippo~\citep{xie2024zippo} shows that a dedicated RGBA VAE is not strictly necessary by encoding replicated alpha with a pretrained RGB VAE, but it still models appearance and opacity with separate latent states. These observations motivate a simpler question: can a pretrained RGB generator learn opacity directly in its existing latent space, using generative supervision alone?

We answer this question with \method, a lightweight framework that augments a pretrained RGB generator with a shared-latent alpha pathway. We insert LoRA modules~\citep{hu2022lora} into the generative transformer and attach a trainable alpha branch to the frozen RGB VAE decoder. During training, flow-matching supervision updates the transformer LoRA, while opacity supervision passes through the frozen decoder features to shape the same predicted latent. Within the decoder, information flows only from the RGB pathway to the alpha branch; the alpha state is never fed back into the RGB pathway. Consequently, for any fixed latent, the augmented decoder produces exactly the same RGB output as the original pretrained decoder, while predicting alpha from the RGB feature hierarchy. RGB and alpha are thus obtained from one latent and one augmented decoder traversal, without an RGBA autoencoder or a separate alpha latent.

Learning opacity also requires an objective that reflects its role in compositing. Alpha is not merely a bounded image channel: $a=0$ and $a=1$ denote exact transparency and opacity, intermediate values encode partial visibility, and $(1-a,a)$ gives the complementary background--foreground mixing weights. We therefore introduce \emph{Endpoint-Aware Alpha Loss (EA Loss)}. EA Loss represents this compositing pair on the Bernoulli simplex and uses its Fisher--Rao coordinate~\citep{miyamoto2024fisherrao} to parameterize the full opacity range. An endpoint-compatible Fenchel--Young formulation represents exact endpoints with finite scores and retains corrective gradients when a prediction saturates at the wrong endpoint. The same complementary weights further separate background- and foreground-associated errors before root aggregation, increasing the relative contribution of localized defects such as thin structures, small holes, and background leakage. We characterize these properties theoretically and evaluate the complete objective against standard alpha $L_1$ supervision.

The paper makes three contributions.
\textbf{(1) Shared-latent RGBA extension.}
We show that pretrained RGB generators can acquire text- and image-conditioned RGBA generation without RGBA reconstruction pretraining. A unidirectional RGB-to-alpha decoder branch lets opacity supervision shape the generative latent while preserving the pretrained RGB decoding function for every fixed latent.
\textbf{(2) Endpoint-aware opacity supervision.}
We derive EA Loss from the complementary compositing weights of alpha. The objective accommodates exact endpoints and continuous partial transparency while emphasizing spatially localized alpha errors, and we establish its key optimization and error-control properties.
\textbf{(3) Evaluation across tasks and backbones.}
We instantiate \method on multiple pretrained RGB generators and evaluate it in both text-to-RGBA and image-to-RGBA settings, where it achieves competitive generation quality without an RGBA autoencoder or a separate alpha latent.

\section{Related Work} \label{sec:related} \noindent\textbf{RGB image generation.} Modern diffusion and flow-based image generators learn powerful text- and image-conditioned visual priors in pretrained RGB latent spaces~\citep{rombach2022highresolution,peebles2023scalable,wu2025qwenimage,wu2026visualgeneration}. These models provide strong appearance synthesis and editing capabilities, but their outputs remain conventional three-channel images whose content is tied to an RGB background. Parameter-efficient training methods such as LoRA~\citep{hu2022lora} further show that new generative capabilities can be introduced while largely reusing a pretrained backbone. Our work builds on this premise, extending pretrained RGB generators with opacity generation rather than replacing their learned RGB representation.

\noindent\textbf{RGBA generation and representations.} Image matting estimates opacity for content already observed in an RGB image, classically with auxiliary spatial guidance such as a trimap or mask~\citep{xu2017deepmatting}. It therefore differs from RGBA generation: matting operates on visible image evidence rather than semantic generation instructions, does not jointly synthesize appearance and opacity, and typically recovers the visible matte rather than completing object layers behind occlusions. Existing native RGBA methods adopt different transparency representations. LayerDiffuse learns a latent transparency representation and decoder through reconstruction training~\citep{zhang2024layerdiffuse}, while AlphaVAE and Qwen-Image-Layered train dedicated four-channel autoencoders~\citep{wang2025alphavae,yin2025qwenlayered}. These approaches introduce an additional reconstruction-trained representation to support transparency generation. Zippo takes a different route by reusing a pretrained RGB VAE and representing alpha through replicated RGB channels~\citep{xie2024zippo}; however, color and matte are still modeled with separate latent states. These works show that transparency can be incorporated through substantially different representation choices. \method instead learns RGBA generation in the original RGB latent space while retaining the pretrained RGB autoencoder.

\section{\method}
\label{sec:method}

Our central insight is that RGBA generation can be learned in the original RGB latent space. \method realizes this with DiT LoRA modules and a lightweight alpha branch embedded in the frozen RGB VAE decoder. Within the alpha branch, features are shared unidirectionally from RGB to alpha, preserving the original RGB decoding process for any fixed latent, while gradients from the alpha loss propagate through the decoder to update the DiT LoRA parameters. Our Endpoint-Aware Alpha Loss (EA Loss) further provides supervision for both exact endpoints and continuous opacity. Figure~\ref{fig:overview}  summarizes the framework.

\begin{figure}[t]

\centering
\includegraphics[width=\linewidth]{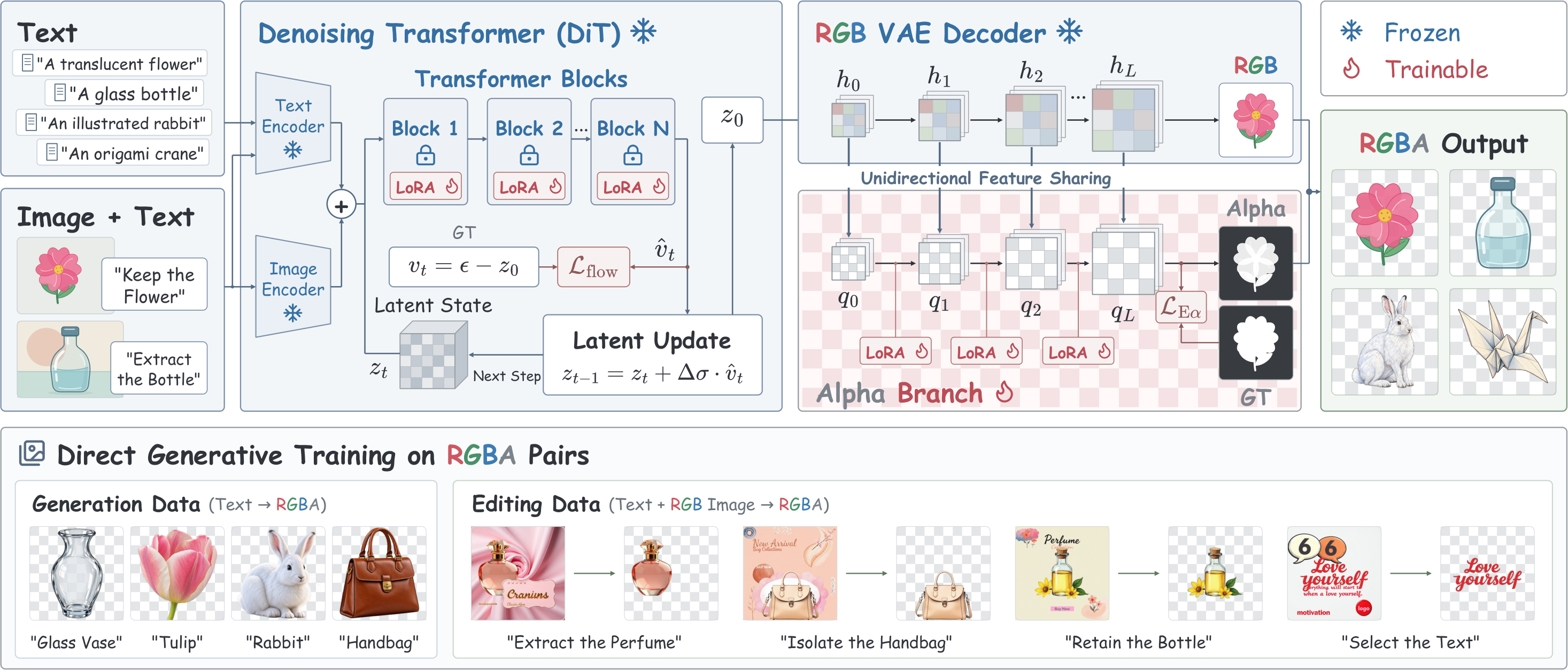}
\par
\caption{Overview of \method.  Transformer LoRA and a VAE-embedded alpha branch learn RGBA generation through generative training. Unidirectional feature sharing keeps RGB decoding for any fixed latent.}
\label{fig:overview}
\par
\end{figure}

\subsection{Preliminaries}
\label{sec:preliminary}

\textbf{RGBA image.} A transparent image is represented as an RGBA layer $I_t=(I_c,I_\alpha)$, where $I_c\in[0,1]^{N\times3}$ is the \emph{straight} foreground color and $I_\alpha\in[0,1]^N$ is opacity, with $N=HW$.  Given a background $B\in[0,1]^{N\times3}$, compositing in a common linear color space is
\begin{equation}
 C_B(I_t)=I_\alpha\odot I_c+(1-I_\alpha)\odot B
         =I_p+(1-I_\alpha)\odot B,
 \label{eq:compositing}
\end{equation}
where $I_p=I_\alpha\odot I_c$ is the \emph{premultiplied} foreground and opacity is broadcast over color channels.  The endpoints $I_{\alpha,i}=0$ and $1$ represent complete transparency and opacity; intermediate values represent partial transparency. For brevity, we denote $I_\alpha$ by $a$ below.

\textbf{RGB training targets.}
As implied by the compositing equation in \eqref{eq:compositing}, a composite does not uniquely determine its foreground color--opacity pair. In particular, when opacity vanishes, the stored RGB values are physically undefined and do not affect the composite~\citep{zhang2024layerdiffuse}. However, generative training in the pretrained RGB latent space requires a well-defined RGB target over the entire image, including transparent regions; leaving these values unconstrained would make the training target underdetermined. We therefore use the multiscale color-padding convention of LayerDiffuse to define a dense RGB target:
\begin{equation}
\widetilde I_c=\operatorname{Pad}(I_c,a),\qquad
z_0=E(\widetilde I_c).
\label{eq:padding}
\end{equation}
The padding propagates foreground color from coarse to fine into transparent regions, providing a structured RGB target from which the frozen RGB encoder $E$ defines the target latent $z_0$.

\subsection{Extending RGB Generation with Opacity}
\label{sec:shared-state}

Our goal is to endow a pretrained RGB generator with opacity generation without introducing a separate transparency latent or altering its pretrained RGB decoding function. Our key insight is that opacity can be learned directly in the original RGB latent space: alpha supervision is propagated back to the generator, so the sampled latent is jointly shaped by appearance and opacity. Opacity is therefore learned as part of generation rather than estimated post hoc from a completed RGB image.

Let $\theta_0$ denote the pretrained generator parameters, $\psi$ its trainable LoRA parameters, and $D$ the frozen RGB decoder. The extended model produces
\begin{equation}
z=G_{\theta_0,\psi}(\epsilon,c),\qquad \widehat I_c=D(z),\qquad s=H_\phi(z;D),
\label{eq:lift}
\end{equation}
where $\epsilon\sim\mathcal N(0,I)$ is sampling noise and $H_\phi$ is the trainable alpha branch producing opacity scores $s\in\mathbb{R}^N$. RGB and opacity thus originate from the same sampled latent $z$, with no separately encoded or sampled alpha latent. Because gradients from the alpha objective pass through the frozen decoder to $\psi$, opacity supervision also shapes the latent produced by the generator.

The alpha branch runs alongside the spatial hierarchy of $D$ and reuses the RGB features produced during the same decoder traversal. Let $h_\ell$ denote the RGB feature at spatial stage $\ell$, and let $q_\ell$ denote the spatially aligned alpha state. At the decoder bottleneck, we initialize $q_0=A_\phi(h_0)$ and refine it with cross-attention using alpha queries and RGB keys and values. At each subsequent stage,
\begin{equation}
\bar q_\ell=q_\ell+\gamma_\ell B_{\phi,\ell}^{3\times3}\left(A_{\phi,\ell}^{1\times1}(h_\ell)\right),\qquad q_{\ell+1}=U_{\phi,\ell}(\bar q_\ell),
\label{eq:alpha-stage}
\end{equation}
where $A_{\phi,\ell}^{1\times1}$ reduces the $C_\ell$ RGB channels to a low-rank bottleneck of width $r_\ell\ll C_\ell$, and $B_{\phi,\ell}^{3\times3}$ spatially projects the compressed feature into the alpha state. The learned gate $\gamma_\ell$ controls the contribution of the RGB feature at each stage, while $U_{\phi,\ell}$ contains the alpha-only residual refinement, channel transition, and learned upsampling that advance the alpha state to the next spatial scale. A final spatial head maps the last alpha state to $s$; Section~\ref{sec:fisher} defines the corresponding opacity mapping. This reuse keeps the extension lightweight: only the alpha stream and low-rank RGB-to-alpha projections are added, with no additional RGB-decoder forward.

Crucially, all cross-branch dependencies are directed from RGB to alpha. The RGB hierarchy remains
\begin{equation}
h_{\ell+1}=d_\ell(h_\ell),
\label{eq:rgb-update}
\end{equation}
where $d_\ell$ is the original frozen RGB operation and has no dependence on $q_\ell$ or $\phi$.

\begin{proposition}[Fixed-latent RGB preservation]
\label{prop:triangular}
For every fixed latent $z$ and alpha parameter setting $\phi$, the augmented decoder's RGB output satisfies
\begin{equation}
D_{\rm aug}^{\rm RGB}(z;\phi)=D(z),\qquad \frac{\partial D_{\rm aug}^{\rm RGB}(z;\phi)}{\partial\phi}=0.
\label{eq:rgb-preservation}
\end{equation}
\end{proposition}

This follows directly from  \eqref{eq:rgb-update}: RGB decoding is independent of the alpha state, so for any fixed latent its output is identical to that of the pretrained decoder. Meanwhile, $D$ remains differentiable, allowing gradients from the alpha loss to propagate through it and supervise the generator.

\begin{figure}[tb]

\centering
\includegraphics[width=\linewidth]{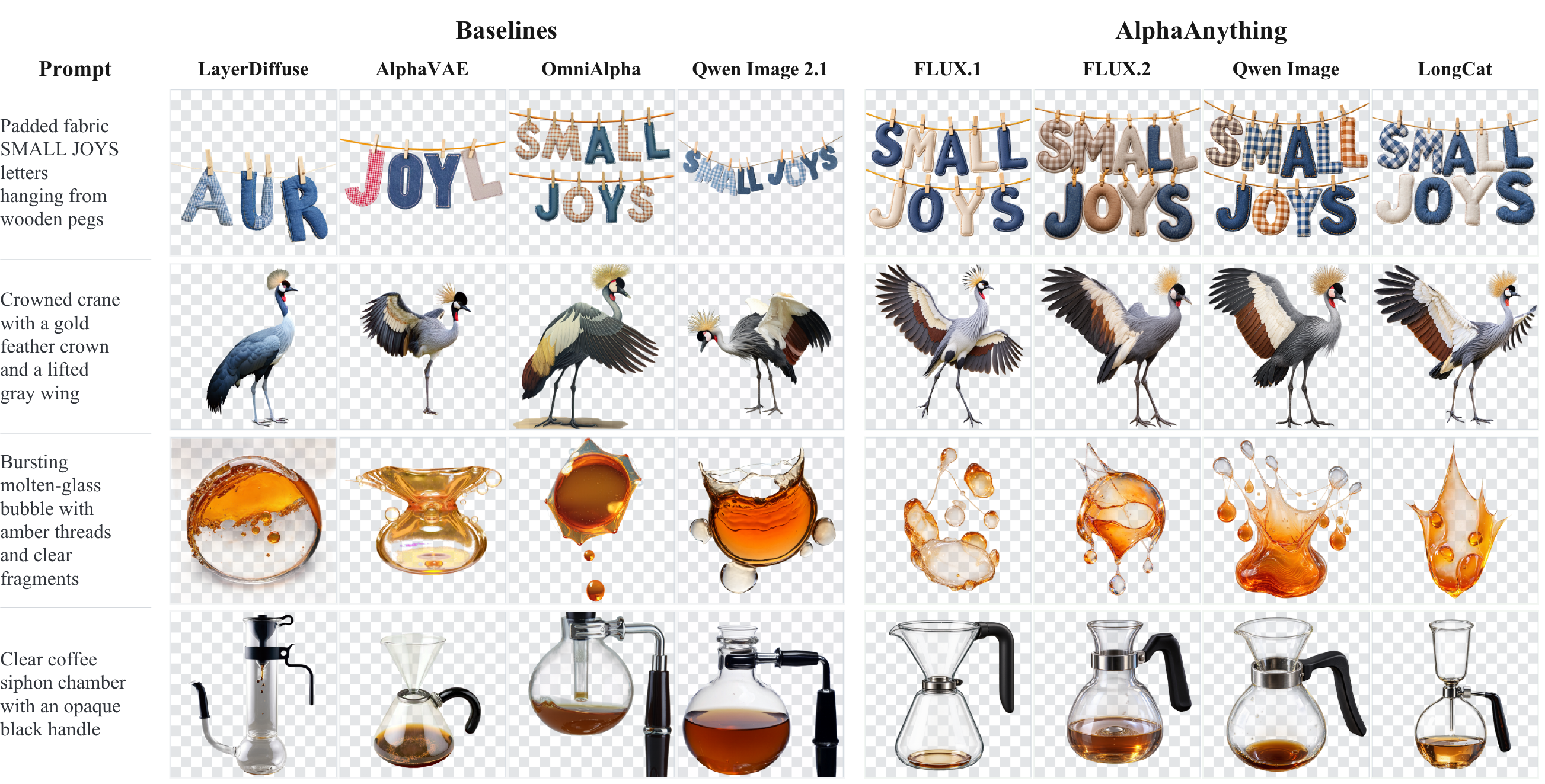}
\par
\caption{Qualitative text-to-RGBA comparison. \method generates detailed lettering and feathers while capturing glass transparency.}
\label{fig:t2rgba-backbone-comparison}
\end{figure}

\subsection{Endpoint-Aware Alpha Loss}
\label{sec:fisher}

Alpha determines foreground--background mixing through the complementary weights $(a,1-a)$. Our central idea is to construct both the pixelwise discrepancy and its spatial aggregation from this compositing pair. We use the Fisher--Rao geometry of the pair to define an opacity coordinate, and reuse its two components to weight the pixelwise gaps before separate root reductions.

\textbf{Compositing-pair geometry.} Equation~\eqref{eq:compositing} associates each opacity value $a$ with the background--foreground weight pair $(1-a,a)$. This pair lies on the Bernoulli simplex, with $a=0$ and $a=1$ at its two vertices and partial opacity in its interior. Its Fisher information is $\mathcal I(a)=1/[a(1-a)]$ for $0<a<1$. The corresponding normalized Fisher--Rao arc-length coordinate~\citep{miyamoto2024fisherrao} and inverse are
\begin{equation}
 u=\chi(a)=\frac{2}{\pi}\arcsin\sqrt a,\qquad
 a=f(u)=\sin^2\!\left(\frac{\pi u}{2}\right).
 \label{eq:fisher-coordinate}
\end{equation}
Both extend continuously to the two simplex vertices. This geometry is non-uniform in physical alpha: near transparency, $\chi(\delta)^2=4\delta/\pi^2+O(\delta^2)$ as $\delta\downarrow0$, with symmetric behavior near opacity, while it remains locally smooth in the interior. Consequently, small residual foreground mass near transparency and residual background mass near opacity receive greater separation in the supervision space, while genuine partial opacity remains represented continuously.

\textbf{Score supervision in the pair coordinate.} To predict this coordinate while covering both the interior and the vertices of the simplex, the alpha branch outputs an unconstrained score $s\in\mathbb{R}$. We map it to the Fisher--Rao coordinate by $p(s)=\clip((s+1)/2,0,1)$ and predict $\widehat a=f(p(s))$. Finite scores $s\leq-1$ and $s\geq1$ therefore represent the two compositing vertices exactly. Direct regression through this clipped mapping, however, would have zero gradient when a prediction saturates at the wrong vertex. We instead supervise $s$ with a Fenchel--Young loss~\citep{blondel2020fenchelyoung} for the target coordinate $u=\chi(a)$:
\begin{equation}
 g(s,u)=
 \begin{cases}
 u(u-1-s), & s\leq-1,\\
 ((s+1)/2-u)^2, & -1<s<1,\\
 (1-u)(s-u), & s\geq1.
 \end{cases}
 \label{eq:fy-gap}
\end{equation}
Within the active interval $-1<s<1$, $g$ is exactly the squared coordinate error, while its gradient satisfies $\partial g/\partial s=p(s)-u$ everywhere. Thus an incorrectly transparent prediction receives gradient $-u$, and an incorrectly opaque prediction receives $1-u$, allowing either to leave the saturated region. Correct endpoint predictions remain exactly representable, and every partial-opacity target has the finite exact solution $s=2u-1$.

\textbf{Pair-weighted aggregation.} Let $g_i=g(s_i,\chi(a_i))$ denote the per-pixel gap over $N$ pixels. We use the target compositing weights to separate foreground- and background-associated contributions of the same pixelwise gap. Applying a square-root reduction to each contribution gives:
\begin{equation}
 R_0=\sqrt{\frac1N\sum_i(1-a_i)g_i},\qquad
 R_1=\sqrt{\frac1N\sum_i a_i g_i},\qquad
 \ell_{\mathrm E\alpha}(s,a)=R_0+R_1.
 \label{eq:root-risk}
\end{equation}
The two weights assign each error according to its target background and foreground mass: transparent pixels contribute to $R_0$, opaque pixels to $R_1$, and partial-opacity pixels continuously to both. For a binary target, suppose nonzero gaps are confined to a fraction $\rho$ of pixels in one endpoint group, with mean gap $\bar g$ over that region. The corresponding root contribution is $\sqrt{\rho\bar g}$, whereas direct averaging gives $\rho\bar g$. Thus, for a fixed mean gap, the root contribution decreases with the square root of the region's area fraction. We select a zero generalized gradient for an exactly zero root. Since $(1-a_i)+a_i=1$ and each gap vanishes exactly at its target, $\ell_{\mathrm E\alpha}=0$ if and only if $\widehat a=a$.

\subsection{Joint Generative Training and Inference}
\label{sec:joint-training}

We jointly optimize the transformer LoRA $\psi$ and alpha branch $\phi$, while keeping the pretrained transformer parameters $\theta_0$ and the RGB VAE encoder and decoder frozen. We follow the backbone's standard flow-matching objective. Given the clean target latent $z_0$ in \eqref{eq:padding}, noise $\epsilon\sim\mathcal N(0,I)$, and noise level $\sigma\in[0,1]$, the noisy latent and target velocity are
\begin{equation}
z_\sigma=(1-\sigma)z_0+\sigma\epsilon,\qquad v^\star=\epsilon-z_0.
\label{eq:flow-path}
\end{equation}
The model predicts
\begin{equation}
\widehat v=v_{\theta_0,\psi}(z_\sigma,\sigma,c),\qquad \widehat z_0=z_\sigma-\sigma\widehat v,
\label{eq:flow}
\end{equation}
where $\widehat z_0$ is the corresponding one-step clean-latent estimate. The alpha branch is applied to $\widehat z_0$ and supervised against the target opacity $a$.

The complete objective is
\begin{equation}
\mathcal L=\mathbb E\!\left[\|\widehat v-v^\star\|_2^2+\lambda_{\mathrm E\alpha}w(\sigma)\ell_{\mathrm E\alpha}\!\left(H_\phi(\widehat z_0;D),a\right)\right],
\qquad
w(\sigma)=\min\!\left\{1,\frac{0.5^2}{\sigma^2+10^{-4}}\right\}.
\label{eq:total-loss}
\end{equation}
The bounded weight $w(\sigma)$ reduces alpha supervision at high noise levels, where the one-step clean-latent estimate is less reliable.

At inference, the backbone produces the final clean latent $z_0$, which is
passed directly to the augmented VAE decoder:
\begin{equation}
 \widehat I_c=D(z_0),\qquad
 \widehat a=f\!\left(p(H_\phi(z_0;D))\right).
 \label{eq:inference}
\end{equation}

\section{Experiments}
\label{sec:experiments}

We evaluate whether pretrained RGB generators can acquire RGBA capability through generative training alone. All \method models keep the RGB VAE encoder and decoder frozen and jointly train the transformer LoRA and alpha head with flow matching and EA Loss, without any RGBA reconstruction training. For text-to-RGBA, we construct and curate 8,963 RGBA training samples using an image generation model and evaluate \method on FLUX.1-dev~\citep{flux2024}, FLUX.2 Klein Base 9B~\citep{flux-2-2025}, Qwen-Image~\citep{wu2025qwenimage}, and LongCat-Image~\citep{meituanlongcatteam2025longcatimagetechnicalreport}. For image-to-RGBA, we use PrismLayersPro~\citep{chen2025prismlayers} with corrected prompt annotations, and evaluate on FLUX.2 Klein Base 9B~\citep{flux-2-2025} and Qwen-Image-Edit~\citep{wu2025qwenimage}.

All \method variants are trained for 20K optimizer updates with an effective global batch size of 16, a learning rate of $10^{-4}$ for both the DiT LoRA and alpha branch, and DiT LoRA rank 64. Trainable parameter counts for each backbone are reported in Appendix~\ref{app:efficiency}; detailed training configurations, evaluation protocols, and metric definitions are provided in Appendix~\ref{app:protocol}.

\subsection{Main Results}

We compare \method with existing text-to-RGBA and image-to-RGBA methods.
We evaluate on two test sets constructed for this study: \textbf{RGBA-Text}, containing 294 text prompts, and \textbf{RGBA-Image}, containing 735 image-conditioned examples. We also evaluate text-to-RGBA generation on the 402 prompts in AlphaTest~\citep{wang2025alphavae}. Bold and underlined table entries mark the best and second-best values.

\textbf{Text-to-RGBA generation.}
We evaluate text-to-RGBA generation on RGBA-Text and AlphaTest using CLIPScore~\citep{hessel2021clipscore}, Aesthetic Score~\citep{schuhmann2022laion}, and ImageReward~\citep{xu2023imagereward}, averaging results on black and white composites. On AlphaTest, we additionally report FID~\citep{heusel2017gans} against the reference images. We compare four \method backbones with AlphaVAE (based on FLUX.1-dev), LayerDiffuse, OmniAlpha (based on Qwen-Image-Edit), and Qwen-Image-2.1~\citep{qwen2026image21}.

\begin{table}[tb]

\centering
\caption{Quantitative results of text-to-RGBA generation on RGBA-Text.}
\label{tab:t2rgba-manual}

\small
\renewcommand{\arraystretch}{1.25}
\setlength{\tabcolsep}{8pt}
\AlphaFitTable{%
\begin{tabular}{@{}llccc@{}}
\toprule
\multicolumn{2}{c}{\AlphaTableHead{Method / backbone}} & \AlphaTableHead{CLIPScore\,$\uparrow$} & \AlphaTableHead{Aesthetic\,$\uparrow$} & \AlphaTableHead{ImageReward\,$\uparrow$} \\
\midrule
\multicolumn{2}{c}{LayerDiffuse} & 0.8542 & 5.4505 & 0.5921 \\
\multicolumn{2}{c}{AlphaVAE} & 0.8306 & 5.2679 & 0.5159 \\
\multicolumn{2}{c}{OmniAlpha} & 0.8508 & 5.2850 & 0.7535 \\
\multicolumn{2}{c}{Qwen-Image-2.1} & 0.8465 & 5.3011 & 0.6189 \\
\midrule
\multirow{4}{*}{\method} & FLUX.1 & 0.8353 & \underline{5.5256} & 0.7588 \\
 & FLUX.2 & 0.8656 & 5.5125 & \underline{0.9752} \\
 & Qwen Image & \textbf{0.8684} & 5.5240 & 0.9477 \\
 & LongCat & \underline{0.8657} & \textbf{5.5295} & \textbf{1.0013} \\
\bottomrule
\end{tabular}%
}
\end{table}

As shown in Tables~\ref{tab:t2rgba-manual} and~\ref{tab:t2rgba-alphatest}, \method achieves improvements in visual quality and text--image alignment metrics. Additional results on AIM-500~\citep{li2021automaticmatting} are reported in Appendix~\ref{app:t2rgba-aim500}. For qualitative comparisons in Figure~\ref{fig:t2rgba-backbone-comparison}, \method shows higher visual quality. Further matched examples are provided in Appendix~\ref{app:qualitative}.

\begin{table}[tb]

\centering
\caption{Quantitative results of text-to-RGBA generation on AlphaTest.}
\label{tab:t2rgba-alphatest}

\small
\renewcommand{\arraystretch}{1.25}
\setlength{\tabcolsep}{8pt}
\AlphaFitTable{%
\begin{tabular}{@{}llcccc@{}}
\toprule
\multicolumn{2}{c}{\AlphaTableHead{Method / backbone}} & \AlphaTableHead{\mbox{FID\,$\downarrow$}} & \AlphaTableHead{CLIPScore\,$\uparrow$} & \AlphaTableHead{Aesthetic\,$\uparrow$} & \AlphaTableHead{ImageReward\,$\uparrow$} \\
\midrule
\multicolumn{2}{c}{LayerDiffuse} & 81.2141 & \textbf{0.8487} & 5.5169 & 1.1365 \\
\multicolumn{2}{c}{AlphaVAE} & 74.3047 & 0.8214 & 5.5919 & 1.1098 \\
\multicolumn{2}{c}{OmniAlpha} & \underline{72.7426} & 0.8384 & 5.3542 & 1.2072 \\
\multicolumn{2}{c}{Qwen-Image-2.1} & \textbf{72.2720} & 0.8285 & 5.4672 & 1.1578 \\
\midrule
\multirow{4}{*}{\method} & FLUX.1 & 77.3089 & 0.8096 & \underline{5.7105} & 1.2880 \\
 & FLUX.2 & 75.1796 & \underline{0.8442} & 5.6352 & \underline{1.4464} \\
 & Qwen Image & 75.8076 & 0.8342 & 5.6799 & 1.3758 \\
 & LongCat & 74.9846 & 0.8413 & \textbf{5.7171} & \textbf{1.5050} \\
\bottomrule
\end{tabular}%
}
\end{table}

\textbf{Image-to-RGBA generation.}
Table~\ref{tab:i2rgba-manual} compares the two \method backbones with OmniAlpha~\citep{yu2026omnialpha} and Qwen-Image-2.1 on RGBA-Image, using RGB L1, Soft IoU, SAD, LPIPS~\citep{zhang2018perceptual}, Alpha Grad, and Alpha Conn~\citep{rhemann2009mattingbenchmark}. Figure~\ref{fig:i2rgba-qualitative} compares the predicted RGBA outputs and enlarged alpha details.

\begin{figure}[tb]

\centering
\includegraphics[width=\linewidth]{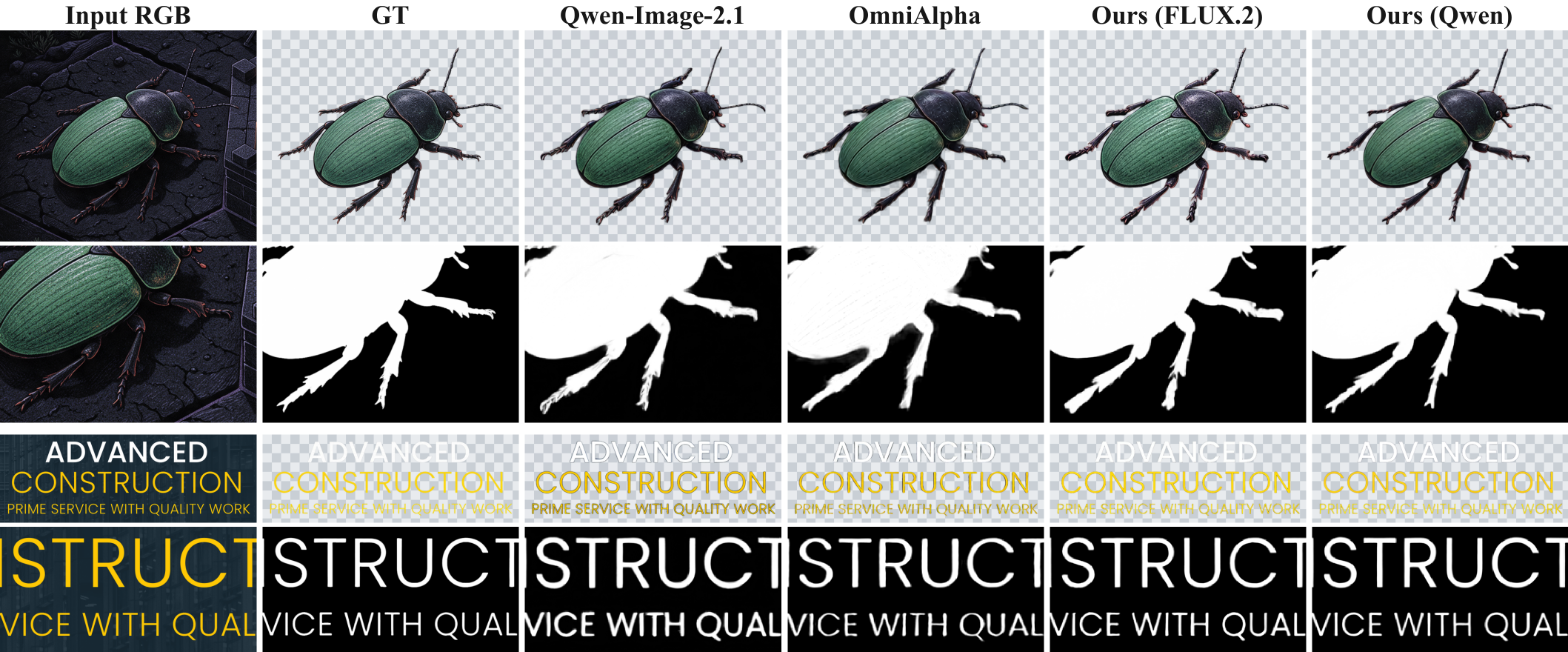}
\par
\caption{Image-to-RGBA comparison. \method produces cleaner alpha boundaries.}
\label{fig:i2rgba-qualitative}
\end{figure}

\begin{table}[tb]
\centering
\caption{Quantitative results of image-to-RGBA evaluation on RGBA-Image.}
\label{tab:i2rgba-manual}

\small
\renewcommand{\arraystretch}{1.25}
\setlength{\tabcolsep}{6pt}
\AlphaFitTable{%
\begin{tabular}{@{}llcccccc@{}}
\toprule
\multicolumn{2}{c}{\AlphaTableHead{Method / backbone}} & \AlphaTableHead{\mbox{RGB L1\,$\downarrow$}} & \AlphaTableHead{\mbox{Soft IoU\,$\uparrow$}} & \AlphaTableHead{\mbox{SAD\,$\downarrow$}} & \AlphaTableHead{\mbox{LPIPS\,$\downarrow$}} & \AlphaTableHead{\mbox{Alpha Grad\,$\downarrow$}} & \AlphaTableHead{\mbox{Alpha Conn\,$\downarrow$}} \\
\midrule
\multicolumn{2}{c}{OmniAlpha} & 0.1052 & 0.6881 & 178.4325 & 0.0911 & 70.8791 & 176.9433 \\
\multicolumn{2}{c}{Qwen-Image-2.1} & \underline{0.0868} & 0.6932 & 181.5245 & 0.0830 & 81.6191 & 166.0596 \\
\midrule
\multirow{2}{*}{\method} & FLUX.2 & 0.1155 & \underline{0.8296} & \underline{46.3572} & \underline{0.0207} & \underline{69.5947} & \underline{46.0840} \\
 & Qwen Image Edit & \textbf{0.0746} & \textbf{0.9223} & \textbf{30.3820} & \textbf{0.0136} & \textbf{27.5221} & \textbf{30.0584} \\
\bottomrule
\end{tabular}%
}

\end{table}

\subsection{Preserving RGB Details}
\label{sec:rgb-detail-preservation}

On the 295 source images in RGBA-Image, we find that RGBA-trained VAEs exhibit lower RGB reconstruction fidelity than their corresponding pretrained backbone VAEs. Compared with the original Qwen-Image-Edit RGB VAE, OmniAlpha's RGBA VAE exhibits visible deviations in fine details, as illustrated in Figure~\ref{fig:vae-rgb-reconstruction}. Quantitatively, LPIPS increases from 0.03807 to 0.05197 and $\Delta E_{00}$ from 1.917 to 2.078. Figure~\ref{fig:i2rgba-rgb-fidelity} further shows that \method better preserves RGB details during RGBA generation.

\begin{figure}[tb]
\centering
\includegraphics[width=\linewidth]{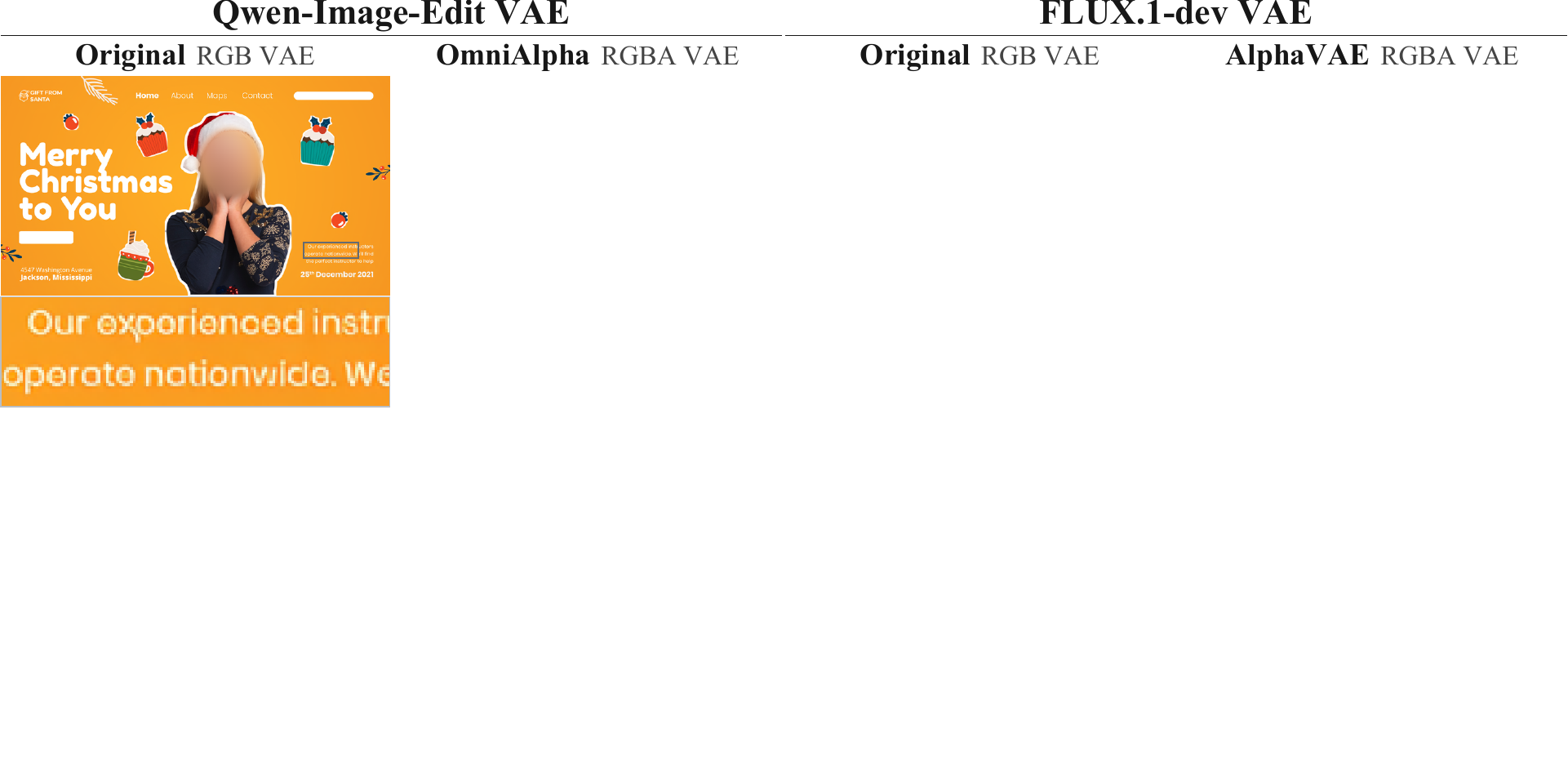}
\caption{RGB reconstruction with pretrained RGB VAEs and their RGBA counterparts. Original RGB VAEs preserve small lettering and colors more faithfully.}
\label{fig:vae-rgb-reconstruction}
\end{figure}

Figure~\ref{fig:rgb-backbone-generation} complements these decoder comparisons with a qualitative generation example using FLUX.1-dev. The original RGB model and \method both render a volumetric nut with clear fabric texture and metallic highlights, while AlphaVAE produces a flatter shape. This example illustrates that \method can retain fine material and structural details exhibited by the original RGB backbone during RGBA generation.

\begin{wrapfigure}{r}{0.43\textwidth}
\centering
\captionsetup{font=small,justification=raggedright,singlelinecheck=false}
\input{figures/rgb_backbone_comparison/figure.tex}
\end{wrapfigure}

We further ablate cross-branch feature sharing using a coupled decoder with additional alpha-to-RGB feedback while freezing all pretrained RGB parameters. As shown in Table~\ref{tab:coupled-decoder}, foreground RGB L1 increases from 0.1155 to 0.1349 on FLUX.2 and from 0.0746 to 0.2435 on Qwen-Image-Edit. These results show that alpha-to-RGB feedback could degrade foreground RGB fidelity, even when the pretrained RGB parameters remain frozen. Our unidirectional RGB-to-alpha design preserves the RGB decoding result for any fixed latent (Proposition~\ref{prop:triangular}).

\WFclear

\begin{figure}[b]
\centering
\includegraphics[width=0.65\linewidth]{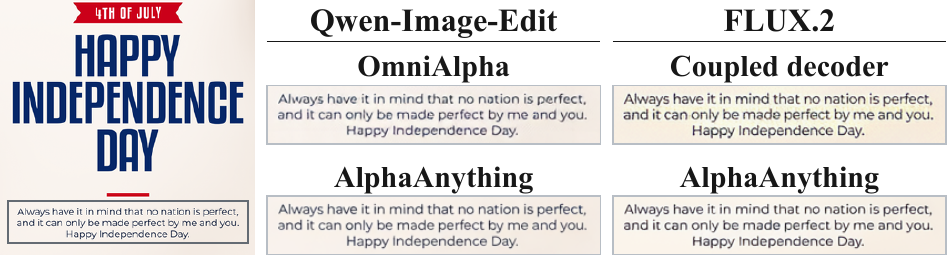}
\caption{RGB detail preservation. \method retains clearer letter shapes than OmniAlpha and the coupled decoder.}
\label{fig:i2rgba-rgb-fidelity}
\end{figure}

\subsection{Effect of Propagating Alpha Supervision to the DiT}
\label{sec:alpha-supervision}
We ablate gradient propagation from the alpha loss to the DiT by detaching \(\widehat z_0\) before alpha decoding while keeping all other training settings fixed. As shown in Table~\ref{tab:alpha-supervision}, allowing alpha-loss gradients to update the DiT leads to more accurate opacity prediction, better alpha boundary and connectivity structure, and improved composite fidelity on both backbones, reflected by higher Soft IoU and lower SAD, LPIPS, Alpha Grad, and Alpha Conn. These results support our central design that opacity supervision should train the generator, rather than being confined to the decoder-side alpha branch.

\begin{table}[t]
\centering
\caption{Coupled-decoder ablation on RGBA-Image.}
\label{tab:coupled-decoder}
\begingroup
\small
\setlength{\tabcolsep}{8pt}
\renewcommand{\arraystretch}{1.25}
\AlphaFitTable{%
\begin{tabular}{@{}lccc@{}}
\toprule
\AlphaTableHead{Decoder} & \AlphaTableHead{RGB L1\,$\downarrow$} & \AlphaTableHead{Alpha Grad\,$\downarrow$} & \AlphaTableHead{Alpha Conn\,$\downarrow$} \\
\midrule
\method FLUX.2 & \textbf{0.1155} & \textbf{69.5947} & \textbf{46.0840} \\
w/ Coupled VAE & 0.1349 & 70.1856 & 59.9648 \\
\midrule
\method Qwen & \textbf{0.0746} & \textbf{27.5221} & 30.0584 \\
w/ Coupled VAE & 0.2435 & 27.8330 & \textbf{28.1564} \\
\bottomrule
\end{tabular}
}
\endgroup

\end{table}

\begin{table}[tb]

\centering
\caption{Ablation on alpha supervision of DiT.}
\label{tab:alpha-supervision}

\small
\renewcommand{\arraystretch}{1.25}
\setlength{\tabcolsep}{8pt}
\AlphaFitTable{%
\begin{tabular}{@{}lccccc@{}}
\toprule
 & \AlphaTableHead{Soft IoU\,$\uparrow$} & \AlphaTableHead{SAD\,$\downarrow$} & \AlphaTableHead{LPIPS\,$\downarrow$} & \AlphaTableHead{Alpha Grad\,$\downarrow$} & \AlphaTableHead{Alpha Conn\,$\downarrow$} \\
\midrule
\method Flux.2 & \textbf{0.8296} & \textbf{46.3572} & \textbf{0.0207} & \textbf{69.5947} & \textbf{46.0840} \\
Stop-grad & 0.8146 & 59.8725 & 0.0248 & 72.3753 & 59.5553 \\
\midrule
\method Qwen & \textbf{0.9223} & \textbf{30.3820} & \textbf{0.0136} & \textbf{27.5221} & \textbf{30.0584} \\
Stop-grad & 0.9014 & 37.0486 & 0.0158 & 35.8853 & 36.5627 \\
\bottomrule
\end{tabular}%
}
\end{table}

\subsection{Effectiveness of the Endpoint-Aware Alpha Loss}
\label{sec:ea-loss-effectiveness}

On FLUX.2 and Qwen-Image-Edit, we compare the complete EA Loss with direct alpha \(L_1\) regression, retaining the same alpha-head architecture and the remaining optimization settings. As shown in Table~\ref{tab:alpha-loss}, EA Loss consistently achieves higher Soft IoU and lower SAD, LPIPS, Alpha Grad, and Alpha Conn on both backbones, with larger gains on Qwen-Image-Edit, demonstrating more effective alpha supervision.

\begin{table}[!htbp]

\centering
\caption{Ablation on EA loss and Alpha L1 loss.}
\label{tab:alpha-loss}

\small
\renewcommand{\arraystretch}{1.25}
\setlength{\tabcolsep}{8pt}
\AlphaFitTable{%
\begin{tabular}{@{}lccccc@{}}
\toprule
 & \AlphaTableHead{Soft IoU\,$\uparrow$} & \AlphaTableHead{SAD\,$\downarrow$} & \AlphaTableHead{LPIPS\,$\downarrow$} & \AlphaTableHead{Alpha Grad\,$\downarrow$} & \AlphaTableHead{Alpha Conn\,$\downarrow$} \\
\midrule
\method Flux.2 (EA) & \textbf{0.8296} & \textbf{46.3572} & \textbf{0.0207} & \textbf{69.5947} & \textbf{46.0840} \\
\method Flux ($L_1$) & 0.8179 & 53.2542 & 0.0226 & 69.9201 & 52.3830 \\
\midrule
\method Qwen (EA) & \textbf{0.9223} & \textbf{30.3820} & \textbf{0.0136} & \textbf{27.5221} & \textbf{30.0584} \\
\method Qwen ($L_1$) & 0.8830 & 49.2356 & 0.0222 & 38.0080 & 48.2089 \\
\bottomrule
\end{tabular}%
}
\end{table}

\section{Conclusions, Limitations and Future work}
\label{sec:discussion}

\method extends pretrained RGB generators to RGBA through generative training alone, without RGBA reconstruction pretraining or a separate alpha latent. Alpha supervision guides RGBA generation in the original RGB latent space, while unidirectional RGB-to-alpha feature sharing preserves the original RGB decoding process for any fixed latent. EA Loss jointly handles exact opacity endpoints, soft transparency, and localized errors within a single objective, outperforming standard alpha \(L_1\) loss in our experiments. Experiments across text-to-RGBA and image-to-RGBA tasks and multiple pretrained backbones show competitive RGBA performance, with improvements in aesthetic quality and text-image alignment.

High-quality RGBA data, widely accepted evaluation benchmarks and metrics are still scarce, which limits how broadly we can evaluate the method. Broader datasets and more comprehensive metrics that jointly assess appearance, opacity, and compositing quality are needed to further characterize data efficiency, generation quality, and generalization across RGBA generation tasks.

\clearpage
\appendix
\raggedbottom
\setlength{\textfloatsep}{14pt plus 2pt minus 2pt}
\setlength{\intextsep}{14pt plus 2pt minus 2pt}
\makeatletter
\setlength{\@fptop}{0pt}
\setlength{\@fpsep}{14pt plus 1pt minus 1pt}
\makeatother
\section{Theoretical Analysis}
\label{app:theory}

This appendix establishes the RGB preservation, gradient transmission, and opacity-supervision properties of the method in Section~\ref{sec:method}. We use $a=I_\alpha$, $u=\chi(a)$, $r=p(s)-u$, and $N=HW$. All square roots denote their nonnegative real values.

\subsection{RGB Preservation and Gradient Flow}
\label{app:triangular}

\begin{proof}[Proof of Proposition~\ref{prop:triangular}]
For a fixed latent $z$, the augmented and original RGB decoders start from the same feature $h_0$. If their RGB states agree at stage $\ell$, both apply the same frozen function $d_\ell$ in \eqref{eq:rgb-update}, without reading the alpha state or its parameters $\phi$. Induction gives identical RGB outputs at every stage. The final RGB output is therefore $D(z)$ for every $\phi$, and its derivative with respect to $\phi$ is zero wherever the derivative exists.
\end{proof}

The joint feature update has the block lower-triangular Jacobian
\begin{equation}
 \frac{\partial(h',q')}{\partial(h,q)}=
 \begin{pmatrix}J_{hh}&0\\J_{qh}&J_{qq}\end{pmatrix}.
 \label{eq:triangular}
\end{equation}
Products of these Jacobians retain the zero upper-right block. The lower-left blocks allow opacity supervision to reach the shared decoder input while preserving the RGB feature updates.

For the actual training path in \eqref{eq:flow}, the one-step estimate is $\widehat z_0=z_\sigma-\sigma\widehat v$, followed by $\overline z_0=\clip(\widehat z_0,-20,20)$. Let $K=\partial\overline z_0/\partial\widehat z_0$ be the diagonal clipping Jacobian and $J_s=\partial H_\phi(\overline z_0;D)/\partial\overline z_0$. Holding $z_\sigma$, $\sigma$, and conditioning fixed, the chain rule gives
\begin{equation}
 \nabla_{\widehat v}\big[w(\sigma)\ell_{\mathrm E\alpha}\big]
 =-\sigma w(\sigma)K^\top J_s^\top\nabla_s\ell_{\mathrm E\alpha}.
 \label{eq:flow-gradient}
\end{equation}
Unclipped latent coordinates have unit derivative through $K$, while clipped coordinates have zero derivative away from the clipping boundaries. The frozen decoder retains its derivatives with respect to its input, so the gradient in \eqref{eq:flow-gradient} propagates through the predicted velocity to the transformer LoRA parameters. This is the gradient path retained by the full model and removed by the stop-gradient ablation.

The sampled latent can change as LoRA is trained; Proposition~\ref{prop:triangular} preserves the original RGB decoding function at each fixed latent.

\subsection{Fisher--Rao Geometry of Opacity}
\label{app:geometry}

For $0<a<1$, the Bernoulli pair $(1-a,a)$ has log-likelihood score $(Y-a)/[a(1-a)]$. Its second moment, using $\Var(Y)=a(1-a)$, gives
\begin{equation}
 \mathcal I(a)=\frac{1}{a(1-a)},\qquad
 {\rm d}\ell^2=\frac{{\rm d}a^2}{a(1-a)}.
 \label{eq:fisher-metric}
\end{equation}
Integrating this line element and normalizing by the total arc length gives the coordinate used in \eqref{eq:fisher-coordinate}:
\begin{equation}
 \chi(a)=\frac1\pi\int_0^a\frac{{\rm d}t}{\sqrt{t(1-t)}}
 =\frac2\pi\arcsin\sqrt a.
\end{equation}
For any differentiable interior path, its length is $\pi\int|{\rm d}\chi(a(t))/{\rm d}t|\,{\rm d}t$. A monotone path attains the minimum, so
\begin{equation}
 d_{\rm FR}(a,b)=\pi|\chi(a)-\chi(b)|.
 \label{eq:fisher-distance}
\end{equation}
Endpoint limits extend this distance continuously to $[0,1]$, with $d_{\rm FR}(0,1)=\pi$. The inverse $f(u)=\sin^2(\pi u/2)$ is a continuous bijection on $[0,1]$.

For an interior target, $\chi'(a)=1/[\pi\sqrt{a(1-a)}]$. At zero, expanding $\arcsin x=x+O(x^3)$ gives
\begin{align}
 [\chi(a+\delta)-\chi(a)]^2
 &=\frac{\delta^2}{\pi^2a(1-a)}+o(\delta^2),
 \qquad 0<a<1,\notag\\
 \chi(\delta)^2&=\frac4{\pi^2}\delta+O(\delta^2),
 \qquad\delta\downarrow0.
 \label{eq:fisher-local}
\end{align}
The identity $\chi(1-a)=1-\chi(a)$ gives the corresponding expansion near opacity one. Thus the squared coordinate distance has first-order cost for small deviations from either opacity endpoint and remains locally quadratic at interior targets, as used in the endpoint-sensitive supervision of Section~\ref{sec:fisher}.

\subsection{Exact Endpoints and Corrective Gradients}
\label{app:fy}

For $\Omega(u)=u^2-u+\iota_{[0,1]}(u)$, the convex conjugate is
\begin{equation}
 \Omega^*(s)=\sup_{v\in[0,1]}\{sv-v^2+v\}
 =\begin{cases}
 0,&s\leq-1,\\
 (s+1)^2/4,&-1<s<1,\\
 s,&s\geq1.
 \end{cases}
 \label{eq:conjugate}
\end{equation}
The unconstrained maximizer is $(s+1)/2$; restricting it to $[0,1]$ gives the map $p(s)$ defined in Section~\ref{sec:fisher}. The conjugate values and first derivatives agree at both joins, and $(\Omega^*)'(s)=p(s)$. Substitution into the Fenchel--Young gap
\begin{equation}
 g(s,u)=\Omega^*(s)+u^2-u-su
 \label{eq:fy-conjugate}
\end{equation}
recovers exactly the three branches used in \eqref{eq:fy-gap}.

\begin{proposition}[Endpoint-compatible coordinate supervision]
\label{prop:fy}
For $u\in[0,1]$ and finite $s$,
\begin{equation}
 g(s,u)\geq[p(s)-u]^2\geq0,\qquad
 g(s,u)=0\ \Longleftrightarrow\ p(s)=u,
 \label{eq:fy-bound}
\end{equation}
and $\partial g(s,u)/\partial s=p(s)-u$.
\end{proposition}

\begin{proof}
For $s\in[-1,1]$, the gap is exactly $(p-u)^2$. Outside this interval,
\begin{equation}
 g(s,u)-(p-u)^2=
 \begin{cases}
 u(-1-s),&s\leq-1,\\
 (1-u)(s-1),&s\geq1,
 \end{cases}
\end{equation}
which is nonnegative. Differentiating the three branches gives $-u$, $(s+1)/2-u$, and $1-u$; these derivatives agree at the joins and equal $p-u$. The lower bound makes $p=u$ necessary for zero gap, and substitution into the corresponding branch proves sufficiency.
\end{proof}

A transparent target has zero gap at every $s\leq-1$, and an opaque target has zero gap at every $s\geq1$. Every soft target has an exact fit at $s=2u-1$. At an incorrect transparent prediction the score derivative is $-u<0$, and at an incorrect opaque prediction it is $1-u>0$. The loss therefore supports finite-score exact endpoints while retaining a corrective derivative under incorrect saturation.

The implementation evaluates the nonnegative branches in \eqref{eq:fy-gap} directly and computes targets and reductions in FP32. The equivalent expression $\chi(a)=(2/\pi)\operatorname{atan2}(\sqrt a,\sqrt{1-a})$ evaluates the target coordinate at both endpoints.

\subsection{Root Aggregation and Error Bounds}
\label{app:root}

\begin{proposition}[Root aggregation and corrective gradients]
\label{prop:root}
For the actual roots in \eqref{eq:root-risk}, let $M=N^{-1}\sum_i g(s_i,u_i)$. Then
\begin{equation}
 \sqrt M\leq\ell_{\mathrm E\alpha}\leq\sqrt{2M},\qquad
 \ell_{\mathrm E\alpha}=0\ \Longleftrightarrow\ \widehat a=a.
 \label{eq:root-bounds}
\end{equation}
Where both roots are positive,
\begin{equation}
 \frac{\partial\ell_{\mathrm E\alpha}}{\partial s_i}
 =\frac{r_i}{2N}\left(\frac{1-a_i}{R_0}+\frac{a_i}{R_1}\right).
 \label{eq:root-gradient}
\end{equation}
Selecting the zero generalized gradient for an exactly zero root gives
\begin{equation}
 r_i\frac{\partial\ell_{\mathrm E\alpha}}{\partial s_i}\geq0.
 \label{eq:corrective-gradient}
\end{equation}
\end{proposition}

\begin{proof}
The identity $R_0^2+R_1^2=M$ gives $(R_0+R_1)^2=M+2R_0R_1\in[M,2M]$. A zero sum of the roots implies every gap is zero, since their per-pixel weights sum to one. Proposition~\ref{prop:fy} and the bijection $f$ then give $\widehat a=a$; the converse follows by substitution.

Write $b_{0i}=1-a_i$ and $b_{1i}=a_i$. For $R_k>0$, differentiation gives
\begin{equation}
 \frac{\partial R_k}{\partial s_i}=\frac{b_{ki}r_i}{2NR_k}.
\end{equation}
Summing the two root derivatives proves \eqref{eq:root-gradient}. The weights are nonnegative, so multiplication by $r_i$ establishes \eqref{eq:corrective-gradient}.

At a zero root, all positively weighted gaps and residuals vanish. Its local expression is a norm of affine or one-sided endpoint residuals, or is identically zero. It is locally Lipschitz, and zero is an admissible generalized gradient at its minimum. The selected zero contribution preserves the corrective sign without evaluating $0/0$.
\end{proof}

\paragraph{Contribution of localized errors.}
For a binary target, suppose all nonzero gaps lie in a set $S$ within one endpoint group $k\in\{0,1\}$, so $a_i=k$ for $i\in S$. Let $m=|S|>0$, $\rho=m/N$, and $\bar g=m^{-1}\sum_{i\in S}g_i$. Then
\begin{equation}
 R_k^2=\frac1N\sum_{i\in S}g_i=\rho\bar g,\qquad
 R_{1-k}=0,\qquad
 \ell_{\mathrm E\alpha}=\sqrt{\rho\bar g},\qquad M=\rho\bar g.
 \label{eq:localized-root}
\end{equation}
For fixed $\bar g$, the root contribution therefore decreases as $\sqrt{\rho}$, while direct averaging of the gaps decreases as $\rho$. This identity applies to both saturated and unsaturated scores and establishes the area dependence described in Section~\ref{sec:fisher}.

\paragraph{Opacity and compositing error.}
\label{app:compositing}
The inverse coordinate satisfies $|f'(u)|\leq\pi/2$, giving $|\widehat a_i-a_i|\leq(\pi/2)|r_i|$. Consequently,
\begin{equation}
 \frac1N\sum_i(\widehat a_i-a_i)^2
 \leq\frac{\pi^2}{4N}\sum_i r_i^2
 \leq\frac{\pi^2}{4}M
 \leq\frac{\pi^2}{4}\ell_{\mathrm E\alpha}^2.
 \label{eq:alpha-error-bound}
\end{equation}
Let $F=I_c$ and assume $F,\widehat F,B\in[0,1]^{N\times3}$ in the common linear color space of \eqref{eq:compositing}. The exact decomposition
\begin{equation}
 C_B(\widehat F,\widehat a)-C_B(F,a)
 =a\odot(\widehat F-F)+(\widehat a-a)\odot(\widehat F-B)
\end{equation}
and $|\widehat F-B|\leq1$ channelwise imply
\begin{equation}
 \frac{\|C_B(\widehat F,\widehat a)-C_B(F,a)\|_2}{\sqrt{3N}}
 \leq\sqrt{\frac{\sum_i a_i^2\|\widehat F_i-F_i\|_2^2}{3N}}
 +\frac\pi2\ell_{\mathrm E\alpha}.
 \label{eq:composite-bound}
\end{equation}
This bound separates foreground RGB error from the opacity contribution controlled by EA Loss and holds for every background in the stated range.

\section{Experimental Protocol}
\label{app:protocol}

This appendix specifies the implementation, training, evaluation, and ablation settings used in Section~\ref{sec:experiments}. All variants use the objective in \eqref{eq:total-loss}, with the stated substitution for the alpha-$L_1$ ablation. Additional numerical results and visual comparisons are reported in Appendices~\ref{app:quantitative} and~\ref{app:qualitative}.

\subsection{Implementation Details}
\label{app:implementation}

\paragraph{RGB targets and native backbone interfaces.}
The color-padding operation in Section~\ref{sec:preliminary} follows the multiscale convention of LayerDiffuse~\citep{zhang2024layerdiffuse}.  We form a pyramid of premultiplied color $I_p=aI_c$ and alpha, initialize the coarsest color from an alpha-weighted mean, and propagate color from coarse to fine.  Each scale adds its premultiplied color to the upsampled coarse color multiplied by remaining transparency.  The procedure supplies a dense three-channel training target without relying on arbitrary stored RGB values in fully transparent regions.  The frozen encoder uses its posterior mode for FLUX and Qwen, while LongCat follows its native training convention of posterior sampling.  Each backbone retains its native convolution dimensionality, feature widths, latent scaling, and token packing.  Conditioning tokens remain clean when required by its image-conditioning interface.

\paragraph{Spatial feature injections.}
At decoder stage $\ell$ with $m_\ell$ RGB feature injections, feature $h_{\ell j}$ updates the aligned alpha state by
\begin{equation}
 q\leftarrow q+\frac{\gamma_\ell}{\sqrt{m_\ell}}\,
       \frac{\eta_\ell}{r_\ell}
       B_{\ell j}^{3\times3}A_{\ell j}^{1\times1}h_{\ell j}.
 \label{eq:injection}
\end{equation}
The learned projection $A_{\ell j}$ reduces RGB channels to rank $r_\ell$, and $B_{\ell j}$ performs spatial mixing into the alpha width.  The factor $\eta_\ell/r_\ell$ is the low-rank projection scaling; $1/\sqrt{m_\ell}$ normalizes the accumulated input by injection count; and the learned stage gate $\gamma_\ell$ controls its contribution.  Channel transitions, residual blocks, and learned upsampling align the alpha state with each RGB stage.

\paragraph{Alpha-branch configuration.}
The five alpha stages use widths $(256,256,192,128,96)$ from the bottleneck toward the RGB output, with two residual blocks per stage.  All RGB-to-alpha projections use rank $r_\ell=64$ and scaling parameter $\eta_\ell=64$.  The learned stage gates are initialized to $(1,0.75,0.5,0.25,0.25)$, and injection-count normalization is enabled.  Bottleneck cross-attention uses eight heads, each with 32 channels.  The final spatial score head has 128 hidden channels.  Dropout is zero throughout the alpha branch.

\paragraph{Initialization.}
Alpha residual outputs and the bottleneck cross-attention output projection are zero initialized.  Spatial injection output projections and the final score convolution use random initialization with standard deviation $10^{-4}$.  This starts residual refinement from the identity while providing trainable feature and score paths.

\subsection{Training and Inference Settings}
\label{app:training-settings}

Text-to-RGBA evaluation uses RGBA-Text, AIM-500~\citep{li2021automaticmatting}, and AlphaTest from AlphaVAE~\citep{wang2025alphavae}. For AIM-500, we use the text prompts released with AlphaVAE. Image-to-RGBA evaluation uses RGBA-Image. Prompt lists, source groups, aspect-ratio policy, conditioning protocol, and filtering rules are fixed within each backbone.

The text-to-RGBA comparison uses FLUX.1-dev, FLUX.2 Klein Base 9B, Qwen-Image-2512, and LongCat-Image; the image-to-RGBA comparison uses FLUX.2 Klein Base 9B and Qwen-Image-Edit-2511.  Each backbone keeps its native conditioning interface, latent layout, encoding, scaling, token packing, and sampler.

\paragraph{Training data and preprocessing.}
Text-to-RGBA training uses 8,963 curated RGBA assets generated using an image generation model.  Image-to-RGBA training uses $\sim$ 180K source-image-target-layer pairs from PrismLayersPro~\citep{chen2025prismlayers}, with corrected English extraction instructions.  Images are processed at a nominal resolution of 1024 using aspect-ratio buckets; the training examples in these runs occupy the $1024\times1024$ bucket.  Center cropping and random horizontal flipping are disabled.  Image-conditioned source and target pairs receive aligned spatial preprocessing.  The RGB target uses color padding, while the original target alpha supplies opacity supervision.

\paragraph{Optimization.}
All models are trained for 20,000 optimizer updates with a per-device batch size of one and one gradient-accumulation step.  The main experiments use 16 GPUs, giving an effective global batch size of 16.  We use AdamW~\citep{loshchilov2019adamw} with learning rate $10^{-4}$ for both the transformer LoRA and alpha branch, $(\beta_1,\beta_2)=(0.9,0.999)$, weight decay $0.01$, and $\epsilon=10^{-8}$.  The learning rate follows cosine decay after 100 warmup updates, and the global gradient norm is clipped at 1.0.  Training uses bfloat16 mixed precision and transformer activation checkpointing.  The original transformer, RGB VAE, and text-encoder parameters remain frozen.  The flow and alpha terms both have coefficient 1.  Alpha supervision uses the time weight in \eqref{eq:total-loss}, with reference noise level $0.5$ and denominator stabilizer $10^{-4}$; the one-step clean-latent estimate is clipped to $[-20,20]$ before decoding.

\paragraph{Transformer LoRA modules.}
All transformer LoRA modules use rank 64, scaling parameter 128, Gaussian initialization, and zero dropout.  Qwen and LongCat apply LoRA to the attention query, key, value, and output projections.  FLUX.1 additionally applies LoRA to the context-stream attention projections and the input and output feed-forward projections in both image and context streams.  FLUX.2 applies LoRA to the attention projections together with the fused query/key/value--MLP input projection and output projection of its single-stream blocks.

\paragraph{Conditioning and noise sampling.}
For text-to-RGBA, explicit and implicit caption variants are selected with equal probability; within either form, the short caption is selected with probability 0.2.  Qwen, FLUX.2, and LongCat drop text conditioning with probability 0.1; image-conditioned training retains the reference image when its text is dropped.  FLUX.1 uses its native guidance embedding with training guidance 1.0.  Qwen and FLUX.2 uniformly sample valid entries of a 1,000-step flow schedule with shift 3, restricting scheduler timesteps to $[1,999]$.  FLUX.1 samples from its native shifted schedule over the same timestep range.  LongCat samples $\sigma=\operatorname{sigmoid}(\xi)$ with $\xi\sim\mathcal N(0,1)$ and then applies the backbone's resolution-dependent time shift.  All backbones use the linear flow path and velocity target in \eqref{eq:flow-path}.  Controlled comparisons keep the sampler, guidance settings, inference seed, conditioning inputs, and sampling-step count fixed.

\paragraph{Inference settings.}
The text-to-RGBA comparisons use the corresponding 20,000-step \method checkpoints. FLUX.1 uses 28 sampling steps and guidance 3.5; FLUX.2, Qwen, and LongCat use 50 steps and guidance 4. The baselines use their recorded inference settings: 25 steps and guidance 7 for LayerDiffuse with SDXL~\citep{podell2024sdxl}, 28 steps and guidance 3.5 for AlphaVAE, 50 steps and guidance 4 for OmniAlpha, and 40 steps and guidance 1.0 for Qwen-Image-2.1. For Qwen-Image-2.1 only, we additionally wrap the shared content prompt in its official RGBA transparency template; using the content prompt alone produced largely opaque outputs ($\alpha\approx1$). The main comparisons use inference base seed 42; Appendix~\ref{app:multi-seed} additionally evaluates the fixed checkpoints with base seeds 43 and 44.

\subsection{Construction of RGBA-Text and RGBA-Image}
\label{app:dataset-construction}

We construct RGBA-Text and RGBA-Image for evaluation only.  Neither set is used to train the \method models.

\paragraph{RGBA-Text.}
RGBA-Text contains 294 English prompts covering nine categories, with the distribution in Table~\ref{tab:rgba-text-categories}.  We follow the training set's category taxonomy while composing new combinations of subjects, materials, colors, poses, and visual styles.  The set spans lettering and graphic assets, natural subjects with detailed boundaries, visual effects, and transparent or translucent materials.  Prompts are selected for content coverage, clarity, and distinctiveness, without imposing equal category sizes or fixed description lengths.  They range from 10 to 95 words, with a median of 53 words; prompts involving visible lettering use English text.

We check both the main prompts and their short summaries against all nonempty training caption variants.  The curation combines normalized exact-text matching, shared 12-word sequence checks, and manual review of TF--IDF nearest neighbors.  Descriptions that too closely resemble training captions are revised.  The final set contains no normalized exact duplicates within the test prompts or against the training text, and no shared contiguous 12-word sequences with training captions.  RGBA-Text contains prompts without paired reference images, so we report reference-free scores on this set.

\begin{table}[!htbp]

\centering
\caption{Category distribution of RGBA-Text.}
\label{tab:rgba-text-categories}

\small
\renewcommand{\arraystretch}{1.2}
\setlength{\tabcolsep}{8pt}
\AlphaFitTable{%
\begin{tabular}{@{}l>{\raggedright\arraybackslash}m{0.60\linewidth}c@{}}
\toprule
\AlphaTableHead{Category} & \AlphaTableHead{Content} & \AlphaTableHead{Prompts} \\
\midrule
Advertising assets & Lettering, titles, labels, buttons, and promotional decorations & 33 \\
Animals & Animal subjects & 33 \\
Visual effects & Light, electricity, smoke, particles, and liquids & 32 \\
Graphics & Decorative shapes, print motifs, and craft textures & 27 \\
Objects & Utensils, tools, and crafted objects & 27 \\
People & People, actions, and interactions & 40 \\
Plants & Flowers, foliage, fruit, and seeds & 29 \\
Products & Product and device presentations & 32 \\
Translucent materials & Transparent or translucent materials, including combinations with opaque parts & 41 \\
\midrule
\textbf{Total} & & \textbf{294} \\
\bottomrule
\end{tabular}%
}
\par
\end{table}

\paragraph{RGBA-Image source images.}
RGBA-Image contains 735 image-conditioned examples from 295 source images.  Each example includes a source image, a target RGBA layer, an English layer description, and an extraction instruction.  The source images come from two pools: selected GenPoster100K posters and images synthesized using an image generation and editing model.  Table~\ref{tab:rgba-image-sources} reports the final distribution by source and input-image category; multiple target layers can share one source image.

For the GenPoster100K pool~\citep{Wang_2025_ICCV}, we first filter the released metadata for adequate resolution, valid layouts, and usable text.  Candidates must have a short side of at least 600 pixels, at least 0.7 megapixels, an aspect ratio between 0.45 and 2.20, and 2--24 annotated layout elements.  We reject candidates with malformed layouts, CJK text, or template placeholders, and retain at most one candidate per original design file.  Remaining candidates are ranked using layout metadata, CLIP-based scores~\citep{radford2021clip}, LAION aesthetic scores, MUSIQ~\citep{Ke_2021_ICCV}, and image-quality statistics, followed by visual near-duplicate filtering and OCR screening.  Full-resolution OCR removes remaining CJK text and placeholder copy, leaving 143 posters for layer extraction.

For the synthetic-image pool, we generate 200 source images with the image generation and editing model: 100 stylized design scenes and 100 photorealistic scenes.  Scene prompts vary subject combinations, human presence, spatial layout, camera viewpoint, palette, lighting, and rendering style.  The final retained inputs comprise 99 stylized scenes and 86 photorealistic scenes, contributing 261 and 217 target layers, respectively.

\paragraph{Layer construction, annotation, and manual review.}
We use the same image generation and editing model to construct RGBA reference layers for both source pools.  Posters are decomposed into at most six layers.  For the generated scenes, a vision-language model first identifies one to three visually unambiguous semantic targets per image, and the image model extracts these targets jointly as separate layers.  Reference layers are placed on the source image's canvas, preserving its dimensions and the targets' spatial alignment.  These are generated reference layers that undergo manual filtering.

We use the vision-language model to annotate or re-annotate English descriptions and extraction instructions, followed by manual review against the source images and target RGBA layers.  Descriptions capture the target's visible appearance and spatial relationships; instructions identify the same target and its intended semantic grouping.  Wording and length vary with target complexity.  Review corrects ambiguous targets, mismatched descriptions, and incomplete grouping instructions, including text panels, frames, and attached objects when they belong to the accepted layer.  Each candidate layer is inspected alongside its source image and alpha matte for alpha accuracy, alignment with the description and extraction instruction, and color fidelity.  Layers with inaccurate opacity, prompt mismatches, or visible color deviations are rejected.  The final evaluation set contains 257 poster layers from 110 posters and 478 generated-scene layers from 185 images, for a total of 735 layers from 295 source images.  We use these manually screened reference layers for evaluation, although generative extraction does not guarantee exact pixelwise reconstruction of the source content.

\begin{table}[!htbp]

\centering
\caption{Sources and input categories of RGBA-Image.}
\label{tab:rgba-image-sources}

\small
\renewcommand{\arraystretch}{1.2}
\setlength{\tabcolsep}{8pt}
\AlphaFitTable{%
\begin{tabular}{@{}llcc@{}}
\toprule
\AlphaTableHead{Source} & \AlphaTableHead{Input category} & \AlphaTableHead{Images} & \AlphaTableHead{Examples} \\
\midrule
GenPoster100K & Graphic-design posters & 110 & 257 \\
Synthetic images & Stylized design scenes & 99 & 261 \\
Synthetic images & Photorealistic scenes & 86 & 217 \\
\midrule
\textbf{Total} & & \textbf{295} & \textbf{735} \\
\bottomrule
\end{tabular}%
}
\par
\end{table}

\subsection{Evaluation Protocols and Metrics}
\label{app:evaluation}

\subsubsection{Text-to-RGBA Evaluation}
\label{app:t2rgba-metrics}

We compare against AlphaVAE~\citep{wang2025alphavae}, LayerDiffuse~\citep{zhang2024layerdiffuse}, OmniAlpha~\citep{yu2026omnialpha}, and Qwen-Image-2.1. For AIM-500~\citep{li2021automaticmatting} and AlphaTest, we use the benchmark prompts provided by AlphaVAE. Evaluation uses FID~\citep{heusel2017gans}, CLIPScore~\citep{hessel2021clipscore}, Aesthetic Score~\citep{schuhmann2022laion}, and ImageReward~\citep{xu2023imagereward}. The main-text results are in Tables~\ref{tab:t2rgba-manual} and~\ref{tab:t2rgba-alphatest}; the additional AIM500 comparison is in Appendix~\ref{app:t2rgba-aim500}.

The four appearance metrics are computed separately on lossless black-background and white-background composites, then averaged across the two backgrounds.  CLIPScore uses CLIP ViT-B/32~\citep{radford2021clip} with the caption prefix ``A photo depicts'' and a weight of $2.5$ on the nonnegative text--image cosine similarity.  Aesthetic Score uses the LAION improved aesthetic predictor\footnote{\url{https://github.com/christophschuhmann/improved-aesthetic-predictor}} with CLIP ViT-L/14 and the released \texttt{sac+logos+ava1-l14-linearMSE} weights, and ImageReward uses the original prompt.  Each benchmark is evaluated separately: sample-level scores are averaged uniformly, and FID is computed for each dataset and background before averaging the two values.  Table~\ref{tab:t2rgba-manual} evaluates all 294 RGBA-Text prompts using CLIPScore, Aesthetic Score, and ImageReward. This set has no reference images for FID.

\subsubsection{Image-to-RGBA Evaluation}
\label{app:i2rgba-metrics}

Table~\ref{tab:i2rgba-manual} evaluates RGBA-Image, whose 735 examples each pair a source image and target description with a target RGBA layer.  The RGBA-Image comparison includes OmniAlpha, Qwen-Image-2.1, and the two \method backbones.  Within each inference run, each method contributes one prediction per example, and all six metrics are averaged uniformly over the 735 examples.

For a fair comparison, \method uses exactly the same prompts as OmniAlpha. Each prompt consists of the prefix ``Extract the object described by the text, preserving fine details and transparency:'' followed by the target description. We use these prompts for Qwen-Image-2.1 as well. Figure~\ref{fig:i2rgba-additional} displays each target description, with the shared prefix stated in the caption.

For the six reference-based metrics, predictions are resized to the GT resolution with RGBA LANCZOS resampling, using the paired target RGBA layers in RGBA-Image.  RGB L1 is the mean RGB-channel error at pixels with nonzero GT alpha.  Soft IoU and SAD use normalized alpha values in $[0,1]$ over the whole image.  Soft IoU is $\sum_i\min(\widehat a_i,a_i)/\sum_i\max(\widehat a_i,a_i)$, with a value of zero for an empty union; SAD is $\sum_i|\widehat a_i-a_i|/1{,}000$.

Alpha Grad and Alpha Conn follow the matting benchmark of \citet{rhemann2009mattingbenchmark}, using implementations equivalent to MMEval 0.2.1.\footnote{\url{https://github.com/open-mmlab/mmeval/tree/v0.2.1}} Alpha Grad measures squared differences between gradient magnitudes after independently min--max normalizing each alpha map. We use Gaussian-derivative gradients with $\sigma=1.4$ and divide the summed squared error by $1{,}000$. Alpha Conn measures connectivity disagreement using opacity thresholds in steps of $0.1$, the largest four-connected component of each prediction--reference intersection, and a $0.15$ opacity-difference cutoff in the connectivity map. The summed absolute difference between the connectivity maps is divided by $1{,}000$. Both metrics use float64 arithmetic over the whole image at GT resolution.

LPIPS is computed separately on black and white composites using the stored RGB values, without intermediate image quantization, and then averaged across backgrounds.  It uses the official AlexNet network with version 0.1 and bilinear antialiased resizing to a maximum image side of 512 pixels, without upscaling.  Higher Soft IoU and lower RGB L1, SAD, LPIPS, Alpha Grad, and Alpha Conn indicate better agreement with the reference.

\subsubsection{RGB Reconstruction Evaluation}
\label{app:vae-rgb-protocol}

We compare the RGBA VAEs of OmniAlpha and AlphaVAE with their original Qwen-Image-Edit-2509 and FLUX.1-dev RGB VAEs, respectively.  The evaluation uses all 295 distinct source images in RGBA-Image, counting each input once regardless of how many target layers it has.  Inputs are resized with LANCZOS to a maximum side length of 1024 pixels while preserving aspect ratios, without upscaling.  All models receive the same inputs, edge-padded to multiples of 64; the padding is removed before scoring.  We use deterministic float32 encode--decode with the posterior mode and supply fully opaque alpha to the RGBA encoders.  Decoded RGB values are clipped to $[0,1]$ and evaluated directly, discarding the reconstructed alpha without compositing.

We evaluate RGB reconstruction using LPIPS~\citep{zhang2018perceptual} and CIEDE2000 color difference, $\Delta E_{00}$~\citep{sharma2005ciede2000}. LPIPS uses AlexNet v0.1 on floating-point reconstructions before PNG quantization, at the full evaluation resolution. Both metrics are computed over complete images and averaged equally across inputs. The enlarged regions in Figure~\ref{fig:vae-rgb-reconstruction} illustrate local details within these full-image reconstructions. Results are reported in Appendix~\ref{app:vae-rgb-fidelity}.

\paragraph{Computation of $\Delta E_{00}$.}
We read the saved 8-bit reference and reconstruction PNGs, retain their RGB channels, and divide by 255 to obtain sRGB values in $[0,1]$. The images retain the evaluation resolution above; no additional resizing, region selection, or alpha weighting is applied. We convert each pixel to CIELAB using \texttt{skimage.color.rgb2lab} with the D65 illuminant and the $2^\circ$ standard observer, then apply \texttt{deltaE\_ciede2000} in float64 with $k_L=k_C=k_H=1$. The pixelwise color difference $d_{00}$ follows the standard CIEDE2000 definition of \citet{sharma2005ciede2000}; smaller values indicate closer reconstructed colors.

For image $j$, let $\Omega_j$ be its unpadded pixel domain, and let $\mathbf c_{jp}$ and $\widehat{\mathbf c}_{jp}$ be the reference and reconstructed CIELAB colors at pixel $p$.  The reported dataset score is
\begin{equation}
 \overline{\Delta E}_{00}
 =\frac{1}{N}\sum_{j=1}^{N}
   \frac{1}{|\Omega_j|}\sum_{p\in\Omega_j}
   d_{00}(\mathbf c_{jp},\widehat{\mathbf c}_{jp}),
 \qquad N=295.
 \label{eq:vae-ciede2000-mean}
\end{equation}
Thus we first average pixelwise color differences within each image, then average the image scores with equal weight.  Unlike LPIPS in this experiment, $\Delta E_{00}$ is computed from the saved 8-bit RGB images.

\subsubsection{RGB and RGBA Generation Example}
\label{app:rgb-generation-protocol}

Figure~\ref{fig:rgb-backbone-generation} compares the original FLUX.1-dev RGB model, AlphaVAE, and \method with FLUX.1 using the following shared prompt:
\begin{quote}
A matte blue hexagonal nut has a cream fabric bandage crossing one cracked corner. A copper wrench curves around its lower edge, and two olive check marks float above the open center. Woven threads, brushed metal, and rounded painted forms create a friendly text-free repair-service emblem.
\end{quote}
All three outputs use seed 69, 28 sampling steps, guidance scale 3.5, and $1024\times1024$ resolution. The original RGB image retains its generated background; RGBA outputs are displayed on white. Each complete output canvas is shown without cropping.

\subsection{Ablation Settings}
\label{app:design-ablations}

\paragraph{Shared settings on RGBA-Image.}
Table~\ref{tab:coupled-decoder} and Table~\ref{tab:alpha-supervision} compare the full FLUX.2 Klein Base 9B and Qwen-Image-Edit-2511 models with their corresponding coupled-decoder and stop-gradient variants. Table~\ref{tab:alpha-loss} compares EA Loss and alpha $L_1$ on both backbones using their corresponding full-model controls. Each setting is evaluated on the same 735 target layers after 20K optimizer updates. Inference uses the same input images and target descriptions, base seed 42, 50 sampling steps, guidance 4, and aspect-ratio-preserving resizing to approximately $1024^2$ pixels without cropping.

\paragraph{Unidirectional feature sharing.}
The coupled RGBA decoder adds zero-initialized alpha-to-RGB feedback at the bottleneck and upsampling stages while keeping the pretrained RGB parameters frozen. We report RGB L1 over pixels with nonzero target alpha to measure the effect on foreground color fidelity, together with Alpha Grad and Alpha Conn for opacity structure.

\paragraph{Alpha supervision of the DiT latent state.}
The stop-gradient setting detaches the predicted clean latent before alpha decoding; flow matching still trains the transformer LoRA, and the alpha head continues to train. Soft IoU and SAD measure how preventing alpha supervision from reaching the DiT affects generated opacity; black/white-averaged composite LPIPS measures the corresponding perceptual fidelity. Alpha Grad and Alpha Conn measure gradient and connectivity errors.

\paragraph{Standard alpha $L_1$ loss.}
The alpha-$L_1$ setting replaces the complete EA Loss with the mean absolute error on directly regressed physical opacity:
\begin{equation}
 \ell_{L_1}=\frac{1}{N}\sum_i\left|\widehat a_i^{\rm raw}-a_i\right|.
 \label{eq:alpha-l1-ablation}
\end{equation}
Here $\widehat a^{\rm raw}$ is the linear output of the alpha head, with its final bias initialized to $0.5$.  Training uses this output directly, without the Fisher-coordinate mapping, Fenchel--Young gap, endpoint-group weighting, root aggregation, sigmoid, or output clipping.  Predicted alpha is clipped to $[0,1]$ only when exporting images.  The head's convolutional structure, flow-matching objective, time weight $w(\sigma)$, alpha-loss coefficient, and alpha-to-DiT gradient path are retained.  We report Soft IoU and SAD for opacity accuracy, Alpha Grad and Alpha Conn for gradient and connectivity errors, and black/white-averaged composite LPIPS for perceptual fidelity under the two objectives.

\subsection{Efficiency Accounting}
\label{app:efficiency}

Table~\ref{tab:trainable-parameters} reports the trainable DiT LoRA and alpha-branch parameters for each backbone.
\begin{table}[!htbp]
\centering
\caption{Trainable parameters of \method, in millions. Counts include only the DiT LoRA modules and alpha branch.}
\label{tab:trainable-parameters}
\small
\setlength{\tabcolsep}{8pt}
\renewcommand{\arraystretch}{1.2}
\AlphaFitTable{%
\begin{tabular}{llcc}
\toprule
\AlphaTableHead{Backbone} & \AlphaTableHead{Task} & \AlphaTableHead{DiT LoRA} & \AlphaTableHead{Alpha branch} \\
\midrule
FLUX.1-dev & T2RGBA & 179.306 & 12.436 \\
FLUX.2 Klein Base 9B & T2RGBA, I2RGBA & 111.149 & 12.436 \\
Qwen-Image-2512 & T2RGBA & 94.372 & 12.253 \\
LongCat-Image & T2RGBA & 39.322 & 12.436 \\
Qwen-Image-Edit-2511 & I2RGBA & 94.372 & 12.253 \\
\bottomrule
\end{tabular}%
}
\end{table}

The trainable parameter count is
\begin{equation}
 P_{\rm train}=P_{\rm LoRA}+P_{\rm alpha},
\end{equation}
where the two terms count the transformer LoRA and alpha-branch parameters. Inference adds only the alpha branch and its RGB-to-alpha feature injections to the backbone's denoising trajectory and decoder traversal, as specified in Appendix~\ref{app:implementation}.

\clearpage
\section{Additional Quantitative Results}
\label{app:quantitative}

The metric definitions and evaluation settings for these additional results are given in Appendix~\ref{app:protocol}.

\subsection{Text-to-RGBA Results on AIM500}
\label{app:additional-benchmarks}
\label{app:t2rgba-aim500}

Table~\ref{tab:t2rgba-aim500} evaluates the eight text-to-RGBA methods on the 500 AIM500 prompts released with AlphaVAE~\citep{wang2025alphavae}. All four \method backbones obtain higher Aesthetic Score and ImageReward than the compared baselines. FLUX.1 achieves the highest Aesthetic Score, and FLUX.2 achieves the highest ImageReward. FID and CLIPScore are reported under the same black/white-composite protocol as the AlphaTest comparison.

\begin{table}[!htbp]

\centering
\caption{Text-to-RGBA generation on AIM500.}
\label{tab:t2rgba-aim500}

\footnotesize
\setlength{\tabcolsep}{7pt}
\renewcommand{\arraystretch}{1.2}
\AlphaFitTable{%
\begin{tabular}{@{}llcccc@{}}
\toprule
\multicolumn{2}{c}{\AlphaTableHead{Method / backbone}} & \AlphaTableHead{\mbox{FID\,$\downarrow$}} & \AlphaTableHead{\mbox{CLIPScore\,$\uparrow$}} & \AlphaTableHead{\mbox{Aesthetic\,$\uparrow$}} & \AlphaTableHead{\mbox{ImageReward\,$\uparrow$}} \\
\midrule
\multicolumn{2}{c}{LayerDiffuse} & 76.8254 & \textbf{0.8493} & 5.5621 & 1.0830 \\
\multicolumn{2}{c}{AlphaVAE} & 74.7061 & 0.8334 & 5.5631 & 0.9565 \\
\multicolumn{2}{c}{OmniAlpha} & \underline{72.2913} & 0.8426 & 5.3879 & 1.0430 \\
\multicolumn{2}{c}{Qwen-Image-2.1} & \textbf{69.5794} & 0.8409 & 5.4272 & 1.0541 \\
\midrule
\multirow{4}{*}{\method} & FLUX.1 & 80.3795 & 0.8189 & \textbf{5.7184} & 1.2475 \\
 & FLUX.2 & 75.9445 & \underline{0.8484} & 5.6102 & \textbf{1.3638} \\
 & Qwen Image & 77.8268 & 0.8337 & 5.6646 & 1.2680 \\
 & LongCat & 77.4516 & 0.8428 & \underline{5.6834} & \underline{1.3593} \\
\bottomrule
\end{tabular}%
}
\par
\end{table}

\subsection{RGB Reconstruction Fidelity}
\label{app:vae-rgb-fidelity}

Table~\ref{tab:vae-rgb-fidelity} reports reconstruction errors on the distinct source images in RGBA-Image, under the protocol in Appendix~\ref{app:vae-rgb-protocol}. Both RGBA VAEs have higher LPIPS and CIEDE2000 errors than their corresponding original RGB VAEs. For the Qwen-Image-Edit backbone, LPIPS increases from 0.03807 to 0.05197 and $\Delta E_{00}$ from 1.917 to 2.078; for FLUX.1-dev, the corresponding values increase from 0.01216 to 0.04762 and from 2.109 to 3.240. These measurements complement the reconstruction examples in Figure~\ref{fig:vae-rgb-reconstruction}.

\begin{table}[!htbp]

\centering
\caption{RGB reconstruction fidelity on RGBA-Image.}
\label{tab:vae-rgb-fidelity}

\footnotesize
\setlength{\tabcolsep}{8pt}
\renewcommand{\arraystretch}{1.2}
\AlphaFitTable{%
\begin{tabular}{@{}llcc@{}}
\toprule
\AlphaTableHead{RGB backbone} & \AlphaTableHead{VAE} & \AlphaTableHead{LPIPS\,$\downarrow$} & \AlphaTableHead{$\Delta E_{00}$\,$\downarrow$} \\
\midrule
\multirow{2}{*}{Qwen-Image-Edit} & Original RGB & \textbf{0.03807} & \textbf{1.917} \\
 & OmniAlpha RGBA & 0.05197 & 2.078 \\
\midrule
\multirow{2}{*}{FLUX.1-dev} & Original RGB & \textbf{0.01216} & \textbf{2.109} \\
 & AlphaVAE RGBA & 0.04762 & 3.240 \\
\bottomrule
\end{tabular}%
}
\par
\end{table}

\clearpage
\subsection{Robustness to Prompt Templates}
\label{app:prompt-robustness}

We evaluate how instruction wording affects image-to-RGBA extraction with the FLUX.2 Klein Base 9B and Qwen-Image-Edit-2511 variants of \method on all 735 target layers in RGBA-Image. Each target description is kept verbatim, while the surrounding instruction follows one of the four templates in Table~\ref{tab:prompt-robustness}: the standard OmniAlpha extraction template, a minimal request, a background-removal instruction, or explicit RGBA output requirements. Within each backbone, we fix the 20,000-step checkpoint, source images, resizing policy, 50 sampling steps, and guidance scale 4. Each target retains its seed from the corresponding backbone's standard run with base seed 42; standard predictions are reused from the main comparison.

Table~\ref{tab:prompt-robustness} reports the mean metrics for each template. For FLUX.2, Soft IoU ranges from 0.8276 to 0.8296 and LPIPS from 0.0206 to 0.0214; for Qwen Image Edit, the corresponding ranges are 0.9223--0.9299 and 0.0126--0.0138. All three alternative templates improve Qwen Image Edit's RGB L1 and Soft IoU over the standard template, with the background-removal instruction achieving the best RGB L1, Soft IoU, SAD, and LPIPS among the four templates. SAD shows greater variation, spanning 46.3572--52.8804 for FLUX.2 and 27.8815--33.1917 for Qwen Image Edit. These results support robustness of overlap and perceptual fidelity to the tested instruction changes on both backbones, with greater sensitivity in absolute alpha error.

The table uses the same six metrics as the main image-to-RGBA evaluation: RGB L1, Soft IoU, SAD, LPIPS, Alpha Grad, and Alpha Conn, defined in Appendix~\ref{app:i2rgba-metrics}. Alpha Grad and Alpha Conn for the standard template are taken from the corresponding main-comparison runs.

\begin{table}[!htbp]

\centering
\setlength{\parskip}{8pt}
\caption{Prompt-template robustness of \method with FLUX.2 Klein Base 9B and Qwen-Image-Edit-2511 on RGBA-Image. Each reported metric is averaged over 735 target layers. The placeholder \texttt{\{description\}} contains the same target description in every template.}
\label{tab:prompt-robustness}

\footnotesize
\setlength{\tabcolsep}{6pt}
\renewcommand{\arraystretch}{1.2}
\AlphaFitTable{%
\begin{tabular}{@{}>{\raggedright\arraybackslash}m{0.18\linewidth}>{\raggedright\arraybackslash}m{0.77\linewidth}@{}}
\toprule
\AlphaTableHead{Template} & \AlphaTableHead{Complete instruction} \\
\midrule
Standard\newline (OmniAlpha) & \texttt{Extract the object described by the text, preserving fine details and transparency: \{description\}} \\[4pt]
Minimal & \texttt{Extract the object described below:}\newline
\texttt{\{description\}} \\[4pt]
Remove background & \texttt{Remove the background and all other objects, keeping only the target described below:}\newline
\texttt{\{description\}} \\[4pt]
RGBA output & \texttt{Target layer: \{description\}}\newline
\texttt{Output requirements: RGBA format; transparent background; original position and scale.} \\
\bottomrule
\end{tabular}%
}

\par
\AlphaFitTable{%
\begin{tabular}{@{}lcccccc@{}}
\toprule
\AlphaTableHead{Template} & \AlphaTableHead{RGB L1\,$\downarrow$} & \AlphaTableHead{Soft IoU\,$\uparrow$} & \AlphaTableHead{SAD\,$\downarrow$} & \AlphaTableHead{LPIPS\,$\downarrow$} & \AlphaTableHead{Alpha Grad\,$\downarrow$} & \AlphaTableHead{Alpha Conn\,$\downarrow$} \\
\midrule
\multicolumn{7}{@{}l}{\textbf{FLUX.2 Klein Base 9B}} \\
\midrule
Standard & 0.1155 & 0.8296 & 46.3572 & 0.0207 & 69.5947 & 46.0840 \\
Minimal & 0.1154 & 0.8276 & 52.8804 & 0.0214 & 69.6647 & 52.6148 \\
Remove background & 0.1163 & 0.8289 & 48.5944 & 0.0206 & 69.9193 & 48.3572 \\
RGBA output & 0.1160 & 0.8285 & 52.3816 & 0.0213 & 69.8098 & 52.1291 \\
\midrule
\multicolumn{7}{@{}l}{\textbf{Qwen Image Edit}} \\
\midrule
Standard & 0.0746 & 0.9223 & 30.3820 & 0.0136 & 27.5221 & 30.0584 \\
Minimal & 0.0710 & 0.9262 & 33.1917 & 0.0138 & 27.0243 & 32.8460 \\
Remove background & 0.0691 & 0.9299 & 27.8815 & 0.0126 & 26.0230 & 27.5284 \\
RGBA output & 0.0731 & 0.9250 & 31.6661 & 0.0134 & 26.8173 & 31.3408 \\
\bottomrule
\end{tabular}%
}
\par
\end{table}

\clearpage
\subsection{Robustness Across Inference Seeds}
\label{app:multi-seed}

We evaluate inference base seeds 42, 43, and 44 using the trained checkpoints and evaluation settings of the main comparisons. Table~\ref{tab:multi-seed-robustness} reports the mean and sample standard deviation across the three runs on RGBA-Text, AlphaTest, and RGBA-Image. Seed 42 is the main-table run. FID is computed separately within each run and background before averaging over backgrounds and summarizing across seeds.

All four \method text-to-RGBA backbones outperform the compared baselines in ImageReward for every evaluated seed on both text benchmarks. LongCat obtains ImageReward scores of $1.0040\pm0.0023$ on RGBA-Text and $1.5146\pm0.0094$ on AlphaTest. On RGBA-Image, Qwen Image Edit ranks first on all six metrics in each run, with mean Soft IoU $0.9236\pm0.0023$ and LPIPS $0.0132\pm0.0009$. The per-seed comparisons therefore retain these main observations across the three inference runs.

\begin{table}[!htbp]

\centering
\setlength{\parskip}{8pt}
\caption{Inference-seed robustness. Entries give the mean and sample standard deviation.}
\label{tab:multi-seed-robustness}

\footnotesize
\setlength{\tabcolsep}{5pt}
\renewcommand{\arraystretch}{1.2}
\textbf{RGBA-Text}\par
\AlphaFitTable{%
\begin{tabular}{@{}llccc@{}}
\toprule
\multicolumn{2}{c}{\AlphaTableHead{Method / backbone}} & \AlphaTableHead{CLIPScore $\uparrow$} & \AlphaTableHead{Aesthetic $\uparrow$} & \AlphaTableHead{ImageReward $\uparrow$} \\
\midrule
\multicolumn{2}{l}{LayerDiffuse} & 0.8505 $\pm$ 0.0032 & 5.4525 $\pm$ 0.0032 & 0.5684 $\pm$ 0.0217 \\
\multicolumn{2}{l}{AlphaVAE} & 0.8314 $\pm$ 0.0007 & 5.2758 $\pm$ 0.0070 & 0.5305 $\pm$ 0.0128 \\
\multicolumn{2}{l}{OmniAlpha} & 0.8515 $\pm$ 0.0024 & 5.2981 $\pm$ 0.0114 & 0.7268 $\pm$ 0.0233 \\
\multicolumn{2}{l}{Qwen-Image-2.1} & 0.8492 $\pm$ 0.0026 & 5.3023 $\pm$ 0.0092 & 0.6287 $\pm$ 0.0093 \\
\midrule
\multirow{4}{*}{\method} & FLUX.1 & 0.8354 $\pm$ 0.0035 & 5.5108 $\pm$ 0.0128 & 0.7667 $\pm$ 0.0187 \\
 & FLUX.2 & 0.8667 $\pm$ 0.0018 & 5.5136 $\pm$ 0.0031 & 0.9741 $\pm$ 0.0143 \\
 & Qwen Image & 0.8681 $\pm$ 0.0005 & 5.5287 $\pm$ 0.0114 & 0.9738 $\pm$ 0.0255 \\
 & LongCat & 0.8658 $\pm$ 0.0026 & 5.5284 $\pm$ 0.0038 & 1.0040 $\pm$ 0.0023 \\
\bottomrule
\end{tabular}%
}

\textbf{AlphaTest}\par
\AlphaFitTable{%
\begin{tabular}{@{}llcccc@{}}
\toprule
\multicolumn{2}{c}{\AlphaTableHead{Method / backbone}} & \AlphaTableHead{FID $\downarrow$} & \AlphaTableHead{CLIPScore $\uparrow$} & \AlphaTableHead{Aesthetic $\uparrow$} & \AlphaTableHead{ImageReward $\uparrow$} \\
\midrule
\multicolumn{2}{l}{LayerDiffuse} & 80.8024 $\pm$ 0.9294 & 0.8494 $\pm$ 0.0007 & 5.5161 $\pm$ 0.0009 & 1.1457 $\pm$ 0.0119 \\
\multicolumn{2}{l}{AlphaVAE} & 74.2638 $\pm$ 0.7549 & 0.8227 $\pm$ 0.0014 & 5.5810 $\pm$ 0.0110 & 1.1074 $\pm$ 0.0226 \\
\multicolumn{2}{l}{OmniAlpha} & 72.7440 $\pm$ 0.0383 & 0.8364 $\pm$ 0.0027 & 5.3523 $\pm$ 0.0023 & 1.2086 $\pm$ 0.0057 \\
\multicolumn{2}{l}{Qwen-Image-2.1} & 71.9380 $\pm$ 0.3099 & 0.8313 $\pm$ 0.0027 & 5.4760 $\pm$ 0.0083 & 1.1776 $\pm$ 0.0184 \\
\midrule
\multirow{4}{*}{\method} & FLUX.1 & 77.4103 $\pm$ 0.1232 & 0.8097 $\pm$ 0.0009 & 5.7074 $\pm$ 0.0097 & 1.2975 $\pm$ 0.0094 \\
 & FLUX.2 & 74.8920 $\pm$ 0.3339 & 0.8439 $\pm$ 0.0013 & 5.6347 $\pm$ 0.0045 & 1.4422 $\pm$ 0.0086 \\
 & Qwen Image & 75.8810 $\pm$ 0.2563 & 0.8340 $\pm$ 0.0003 & 5.6853 $\pm$ 0.0088 & 1.3707 $\pm$ 0.0094 \\
 & LongCat & 74.9678 $\pm$ 0.5422 & 0.8425 $\pm$ 0.0012 & 5.7198 $\pm$ 0.0084 & 1.5146 $\pm$ 0.0094 \\
\bottomrule
\end{tabular}%
}

\textbf{RGBA-Image}\par
\AlphaFitTable{%
\begin{tabular}{@{}llccc@{}}
\toprule
\multicolumn{2}{c}{\AlphaTableHead{Method / backbone}} & \AlphaTableHead{RGB L1 $\downarrow$} & \AlphaTableHead{Soft IoU $\uparrow$} & \AlphaTableHead{SAD $\downarrow$} \\
\midrule
\multicolumn{2}{l}{OmniAlpha} & 0.1032 $\pm$ 0.0024 & 0.6886 $\pm$ 0.0005 & 175.7390 $\pm$ 2.3482 \\
\multicolumn{2}{l}{Qwen-Image-2.1} & 0.0909 $\pm$ 0.0037 & 0.6853 $\pm$ 0.0069 & 185.0934 $\pm$ 7.4971 \\
\midrule
\multirow{2}{*}{\method} & FLUX.2 & 0.1155 $\pm$ 0.0001 & 0.8270 $\pm$ 0.0022 & 52.6253 $\pm$ 6.4488 \\
 & Qwen Image Edit & 0.0739 $\pm$ 0.0008 & 0.9236 $\pm$ 0.0023 & 29.3979 $\pm$ 3.0715 \\
\bottomrule
\end{tabular}%
}
\par
\AlphaFitTable{%
\begin{tabular}{@{}llccc@{}}
\toprule
\multicolumn{2}{c}{\AlphaTableHead{Method / backbone}} & \AlphaTableHead{LPIPS $\downarrow$} & \AlphaTableHead{Alpha Grad $\downarrow$} & \AlphaTableHead{Alpha Conn $\downarrow$} \\
\midrule
\multicolumn{2}{l}{OmniAlpha} & 0.0902 $\pm$ 0.0010 & 70.8178 $\pm$ 0.8268 & 174.2699 $\pm$ 2.3331 \\
\multicolumn{2}{l}{Qwen-Image-2.1} & 0.0855 $\pm$ 0.0022 & 81.3381 $\pm$ 1.8358 & 169.7300 $\pm$ 7.6814 \\
\midrule
\multirow{2}{*}{\method} & FLUX.2 & 0.0218 $\pm$ 0.0010 & 70.0745 $\pm$ 0.6561 & 52.3602 $\pm$ 6.4494 \\
 & Qwen Image Edit & 0.0132 $\pm$ 0.0009 & 27.5225 $\pm$ 0.4195 & 29.0806 $\pm$ 3.0700 \\
\bottomrule
\end{tabular}%
}
\par
\end{table}

\clearpage
\subsection{Bootstrap Confidence Intervals}
\label{app:bootstrap}

We estimate input-sampling uncertainty with 10,000 paired bootstrap resamples and 95\% bias-corrected and accelerated (BCa) intervals computed with \href{https://docs.scipy.org/doc/scipy/reference/generated/scipy.stats.bootstrap.html}{SciPy}. Scores are first averaged over the three inference runs and, where applicable, the black/white backgrounds. We resample prompts for RGBA-Text and AlphaTest, and source-image clusters for RGBA-Image, retaining all associated layers and the layer-weighted mean. Every method uses the same resamples and bootstrap seed 42. Table~\ref{tab:bootstrap-confidence-intervals} reports the resulting intervals.

Paired differences compare each \method variant with the baseline having the best observed mean for that metric. LongCat's ImageReward gains over OmniAlpha are $0.2771$ with interval $[0.2282,\,0.3296]$ on RGBA-Text and $0.3060$ with interval $[0.2715,\,0.3435]$ on AlphaTest. All six Qwen Image Edit differences on RGBA-Image favor \method; its Soft IoU gain over OmniAlpha is $0.2350$ with interval $[0.2115,\,0.2605]$. These pointwise 95\% intervals quantify input-sampling uncertainty for the mean across the saved runs.

\begin{table}[!htbp]
\centering
\setlength{\parskip}{8pt}
\caption{Bootstrap results. Entries give the mean and 95\% BCa interval.}
\label{tab:bootstrap-confidence-intervals}
\small
\setlength{\tabcolsep}{4pt}
\renewcommand{\arraystretch}{1.15}
\textbf{RGBA-Text}\par
\AlphaFitTable{%
\begin{tabular}{@{}llccc@{}}
\toprule
\multicolumn{2}{c}{\AlphaTableHead{Method / backbone}} & \AlphaTableHead{CLIPScore $\uparrow$} & \AlphaTableHead{Aesthetic $\uparrow$} & \AlphaTableHead{ImageReward $\uparrow$} \\
\midrule
\multicolumn{2}{l}{LayerDiffuse} & 0.8505 [0.8427, 0.8585] & 5.4525 [5.4035, 5.4958] & 0.5684 [0.4780, 0.6580] \\
\multicolumn{2}{l}{AlphaVAE} & 0.8314 [0.8226, 0.8399] & 5.2758 [5.2279, 5.3239] & 0.5305 [0.4351, 0.6247] \\
\multicolumn{2}{l}{OmniAlpha} & 0.8515 [0.8428, 0.8599] & 5.2981 [5.2519, 5.3415] & 0.7268 [0.6385, 0.8117] \\
\multicolumn{2}{l}{Qwen-Image-2.1} & 0.8492 [0.8409, 0.8578] & 5.3023 [5.2570, 5.3482] & 0.6287 [0.5379, 0.7144] \\
\midrule
\multirow{4}{*}{\method} & FLUX.1 & 0.8354 [0.8270, 0.8436] & 5.5108 [5.4614, 5.5584] & 0.7667 [0.6735, 0.8559] \\
 & FLUX.2 & 0.8667 [0.8582, 0.8755] & 5.5136 [5.4640, 5.5614] & 0.9741 [0.8911, 1.0532] \\
 & Qwen Image & 0.8681 [0.8603, 0.8763] & 5.5287 [5.4800, 5.5751] & 0.9738 [0.8920, 1.0515] \\
 & LongCat & 0.8658 [0.8574, 0.8745] & 5.5284 [5.4784, 5.5768] & 1.0040 [0.9244, 1.0803] \\
\bottomrule
\end{tabular}%
}

\textbf{AlphaTest}\par
\AlphaFitTable{%
\begin{tabular}{@{}llccc@{}}
\toprule
\multicolumn{2}{c}{\AlphaTableHead{Method / backbone}} & \AlphaTableHead{CLIPScore $\uparrow$} & \AlphaTableHead{Aesthetic $\uparrow$} & \AlphaTableHead{ImageReward $\uparrow$} \\
\midrule
\multicolumn{2}{l}{LayerDiffuse} & 0.8494 [0.8436, 0.8554] & 5.5161 [5.4870, 5.5436] & 1.1457 [1.0848, 1.2025] \\
\multicolumn{2}{l}{AlphaVAE} & 0.8227 [0.8168, 0.8290] & 5.5810 [5.5473, 5.6127] & 1.1074 [1.0348, 1.1781] \\
\multicolumn{2}{l}{OmniAlpha} & 0.8364 [0.8304, 0.8426] & 5.3523 [5.3219, 5.3818] & 1.2086 [1.1426, 1.2676] \\
\multicolumn{2}{l}{Qwen-Image-2.1} & 0.8313 [0.8251, 0.8375] & 5.4760 [5.4434, 5.5089] & 1.1776 [1.1140, 1.2379] \\
\midrule
\multirow{4}{*}{\method} & FLUX.1 & 0.8097 [0.8032, 0.8165] & 5.7074 [5.6745, 5.7396] & 1.2975 [1.2368, 1.3543] \\
 & FLUX.2 & 0.8439 [0.8374, 0.8506] & 5.6347 [5.6016, 5.6672] & 1.4422 [1.3865, 1.4915] \\
 & Qwen Image & 0.8340 [0.8276, 0.8404] & 5.6853 [5.6529, 5.7176] & 1.3707 [1.3137, 1.4208] \\
 & LongCat & 0.8425 [0.8363, 0.8491] & 5.7198 [5.6893, 5.7497] & 1.5146 [1.4670, 1.5559] \\
\bottomrule
\end{tabular}%
}

\end{table}

\begin{table}[!htbp]
\ContinuedFloat
\centering
\setlength{\parskip}{8pt}
\caption[]{Bootstrap results (continued). Entries give the mean and 95\% BCa interval for RGBA-Image.}
\small
\setlength{\tabcolsep}{4pt}
\renewcommand{\arraystretch}{1.15}
\textbf{RGBA-Image}\par
\AlphaFitTable{%
\begin{tabular}{@{}llccc@{}}
\toprule
\multicolumn{2}{c}{\AlphaTableHead{Method / backbone}} & \AlphaTableHead{RGB L1 $\downarrow$} & \AlphaTableHead{Soft IoU $\uparrow$} & \AlphaTableHead{SAD $\downarrow$} \\
\midrule
\multicolumn{2}{l}{\multirow{2}{*}{OmniAlpha}} & 0.1032 & 0.6886 & 175.7390 \\
\multicolumn{2}{l}{} & [0.0942, 0.1159] & [0.6608, 0.7145] & [145.8230, 220.1509] \\
\multicolumn{2}{l}{\multirow{2}{*}{Qwen-Image-2.1}} & 0.0909 & 0.6853 & 185.0934 \\
\multicolumn{2}{l}{} & [0.0855, 0.0974] & [0.6638, 0.7056] & [153.2073, 224.3758] \\
\midrule
\multirow{4}{*}{\method} & \multirow{2}{*}{FLUX.2} & 0.1155 & 0.8270 & 52.6253 \\
 & & [0.1070, 0.1260] & [0.8050, 0.8457] & [41.4052, 79.6911] \\
 & \multirow{2}{*}{Qwen Image Edit} & 0.0739 & 0.9236 & 29.3979 \\
 & & [0.0698, 0.0796] & [0.9109, 0.9325] & [19.0433, 60.7663] \\
\bottomrule
\end{tabular}%
}
\par
\AlphaFitTable{%
\begin{tabular}{@{}llccc@{}}
\toprule
\multicolumn{2}{c}{\AlphaTableHead{Method / backbone}} & \AlphaTableHead{LPIPS $\downarrow$} & \AlphaTableHead{Alpha Grad $\downarrow$} & \AlphaTableHead{Alpha Conn $\downarrow$} \\
\midrule
\multicolumn{2}{l}{\multirow{2}{*}{OmniAlpha}} & 0.0902 & 70.8178 & 174.2699 \\
\multicolumn{2}{l}{} & [0.0801, 0.1025] & [62.3785, 82.2091] & [144.4615, 218.6509] \\
\multicolumn{2}{l}{\multirow{2}{*}{Qwen-Image-2.1}} & 0.0855 & 81.3381 & 169.7300 \\
\multicolumn{2}{l}{} & [0.0741, 0.0985] & [70.3746, 95.0267] & [138.1622, 208.3698] \\
\midrule
\multirow{4}{*}{\method} & \multirow{2}{*}{FLUX.2} & 0.0218 & 70.0745 & 52.3602 \\
 & & [0.0189, 0.0265] & [60.2413, 81.8434] & [41.1351, 79.3873] \\
 & \multirow{2}{*}{Qwen Image Edit} & 0.0132 & 27.5225 & 29.0806 \\
 & & [0.0109, 0.0183] & [24.0329, 33.1854] & [18.7263, 60.4309] \\
\bottomrule
\end{tabular}%
}
\end{table}

\FloatBarrier
\subsection{Human Evaluation}
\label{app:human-evaluation}

We conducted a blinded pairwise human evaluation of I2RGBA outputs on RGBA-Image with three experienced annotators. Each annotator evaluated randomly assigned comparisons, viewing the input image, extraction instruction, and two anonymized predictions with randomized left--right positions. Annotators selected the preferred output or a tie based on instruction consistency, visual quality, and the naturalness of transparency and object boundaries. Each \method backbone was compared against Qwen-Image-2.1 and OmniAlpha, without comparisons between our backbone variants.

For each backbone--baseline pair, let $W$, $T$, and $L$ denote wins, ties, and losses for \method, respectively. We report two complementary metrics: the win rate excluding ties,
\begin{equation}
    \mathrm{WR}_{\mathrm{excl}} = \frac{W}{W+L},
    \label{eq:human-eval-win-rate-excluding-ties}
\end{equation}
and the tie-adjusted win rate,
\begin{equation}
    \mathrm{WR}_{\mathrm{adj}} = \frac{W+0.5T}{W+T+L}.
    \label{eq:human-eval-tie-adjusted-win-rate}
\end{equation}
The results are reported in Table~\ref{tab:human-evaluation}.

\begin{table}[!htbp]
\centering
\caption{Pairwise human evaluation of \method on RGBA-Image for I2RGBA. Both win rates are reported as percentages.}
\label{tab:human-evaluation}
\footnotesize
\setlength{\tabcolsep}{6pt}
\renewcommand{\arraystretch}{1.2}
\AlphaFitTable{%
\begin{tabular}{@{}llcc@{}}
\toprule
\multirow{2}{*}{\AlphaTableHead{\method backbone}} & \multirow{2}{*}{\AlphaTableHead{Baseline}} & \multicolumn{2}{c}{Win rate (\%)} \\
\cmidrule(lr){3-4}
 & & Excluding ties & Tie-adjusted \\
\midrule
FLUX.2 Klein Base 9B & Qwen-Image-2.1 & 90.48 & 72.37 \\
FLUX.2 Klein Base 9B & OmniAlpha & 77.78 & 63.51 \\
Qwen-Image-Edit-2511 & Qwen-Image-2.1 & 96.00 & 79.49 \\
Qwen-Image-Edit-2511 & OmniAlpha & 100.00 & 80.56 \\
\bottomrule
\end{tabular}%
}
\end{table}

\FloatBarrier
\subsection{Inference Overhead}
\label{app:inference-overhead}

Since the DiT LoRA weights can be merged into the pretrained weights at inference, the additional computation of \method comes only from the alpha branch, including its RGB-to-alpha feature injections. We profile four backbones using a single NVIDIA GPU per run, with the same GPU model across all measurements. All runs use $1024\times1024$ resolution, batch size 1, BF16 precision, and 50 sampling steps. The RGB control uses the same decoder traversal and input latents as ours, with unnecessary block-input retention and Qwen's temporal feature cache removed. Table~\ref{tab:inference-overhead} reports VAE decoding time, peak allocated GPU memory, and end-to-end generation time.

VAE timings are medians of 30 measurements after 10 warmup iterations; generation timings are medians of three measurements after a three-step warmup. End-to-end generation time includes text encoding, sampling, and decoding, excluding model loading and image saving. We compute the end-to-end overhead by dividing the additional VAE decoding time by the matched control's end-to-end generation time.

\begin{table}[!htbp]
\centering
\caption{Inference overhead of \method across 4 backbones. Peak memory is measured for the VAE in isolation. End-to-end generation time is measured for the matched control, and end-to-end overhead is computed from unrounded timings.}
\label{tab:inference-overhead}
\footnotesize
\setlength{\tabcolsep}{5pt}
\renewcommand{\arraystretch}{1.2}
\AlphaFitTable{%
\begin{tabular}{@{}lcccccc@{}}
\toprule
\multirow{2}{*}{Backbone} &
\multicolumn{2}{c}{VAE decode time (s)} &
\multicolumn{2}{c}{Peak memory (GiB)} &
\multirow{2}{*}{\AlphaTableHead{End-to-end\\generation (s)}} &
\multirow{2}{*}{\AlphaTableHead{End-to-end\\overhead (\%)}} \\
\cmidrule(lr){2-3}\cmidrule(lr){4-5}
& RGB VAE & Ours VAE & RGB VAE & Ours VAE & & \\
\midrule
FLUX.1-dev & 0.149 & 0.247 & 2.297 & 2.389 & 22.103 & 0.445 \\
FLUX.2 Klein Base 9B & 0.148 & 0.249 & 2.299 & 2.386 & 41.212 & 0.244 \\
Qwen-Image-2512 & 0.138 & 0.239 & 2.118 & 2.331 & 44.206 & 0.228 \\
LongCat-Image & 0.148 & 0.247 & 2.297 & 2.389 & 23.154 & 0.427 \\
\bottomrule
\end{tabular}%
}
\end{table}

Across all four evaluated backbones, \method adds less than $0.5\%$ runtime overhead to end-to-end generation.

\FloatBarrier

\section{Additional Qualitative Results}
\label{app:qualitative}

This appendix extends the main-text showcases with additional matched comparisons from RGBA-Text and RGBA-Image, qualitative component ablations, and an exploratory study of RGBA input for editing. The generation, extraction, and ablation comparisons follow the evaluation settings in Appendix~\ref{app:protocol}.

\subsection{Text-to-RGBA Generation}
\label{app:t2rgba-qualitative}

Figures~\ref{fig:t2rgba-additional} and~\ref{fig:t2rgba-additional-2} compare all eight text-to-RGBA methods on eleven additional RGBA-Text prompts. Each case displays the shared content prompt above two rows: four baselines followed by all four \method settings. Method-specific prompt wrappers follow Appendix~\ref{app:training-settings}. All RGBA outputs retain their complete canvases and are displayed on the same checkerboard.

\begin{figure}[p]
\centering
\setlength{\parskip}{6pt}
\captionsetup{skip=5pt}
\includegraphics[width=\AlphaFigureWidth]{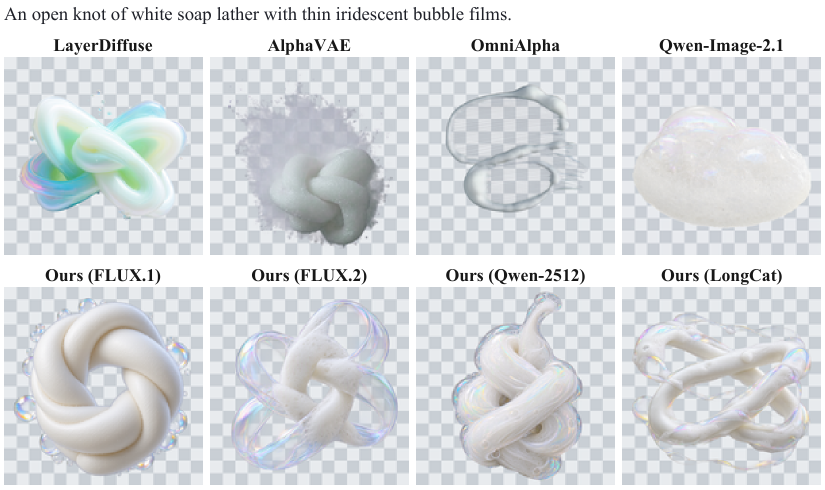}\par
\includegraphics[width=\AlphaFigureWidth]{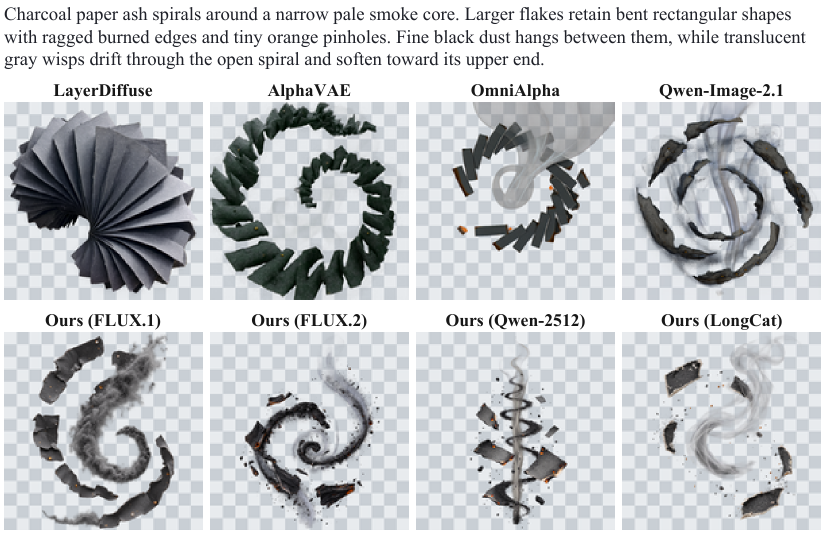}
\caption{Additional qualitative text-to-RGBA comparisons on 6 RGBA-Text prompts (set 1 of 2). Each case shows four baselines in the upper row and all four \method settings in the lower row. Complete RGBA canvases are retained.}
\label{fig:t2rgba-additional}
\end{figure}

\begin{figure}[p]
\ContinuedFloat
\centering
\setlength{\parskip}{6pt}
\captionsetup{skip=5pt}
\includegraphics[width=\AlphaFigureWidth]{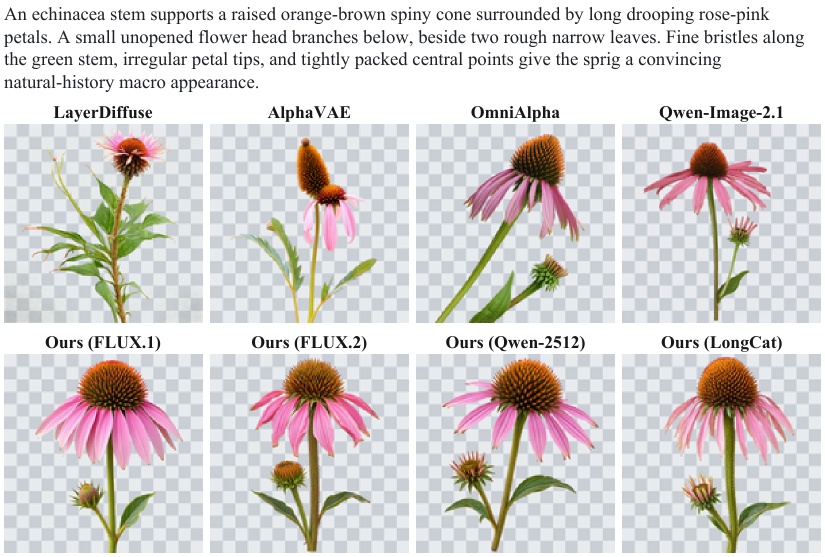}\par
\includegraphics[width=\AlphaFigureWidth]{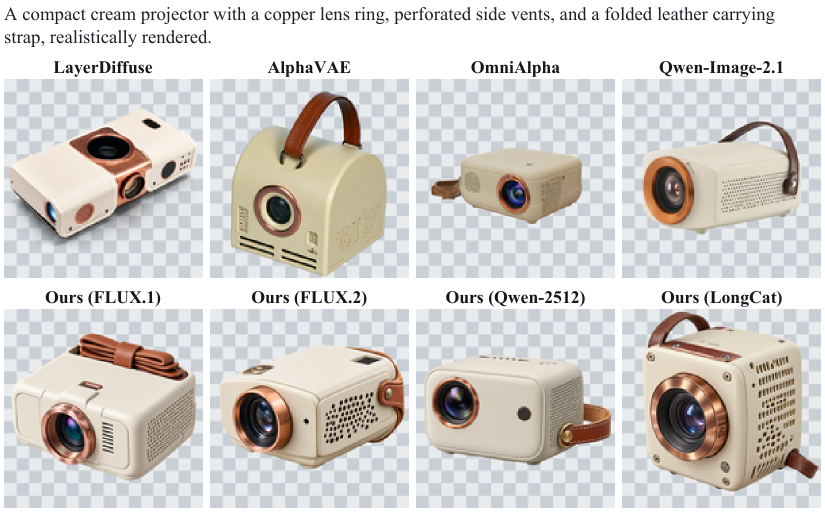}
\caption[]{Additional text-to-RGBA comparisons, set 1 of 2 (continued, cases 3--4 of 6).}
\end{figure}

\begin{figure}[p]
\ContinuedFloat
\centering
\setlength{\parskip}{6pt}
\captionsetup{skip=5pt}
\includegraphics[width=\AlphaFigureWidth]{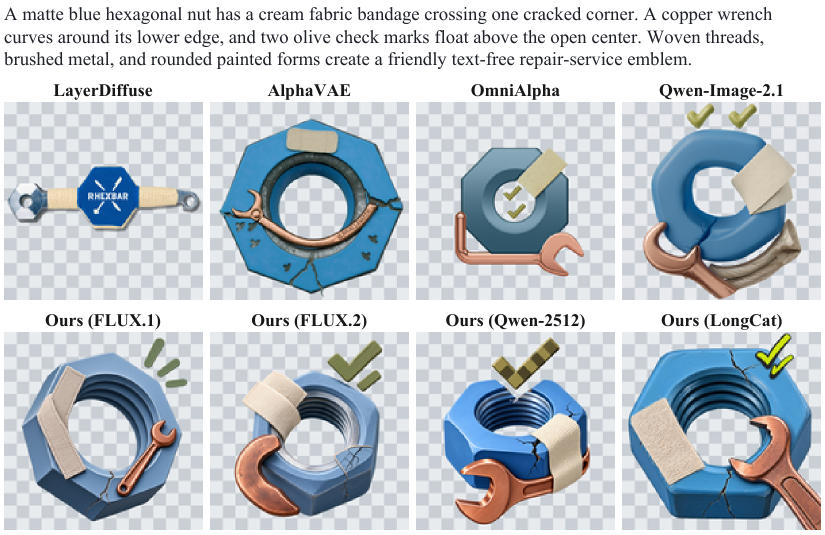}\par
\includegraphics[width=\AlphaFigureWidth]{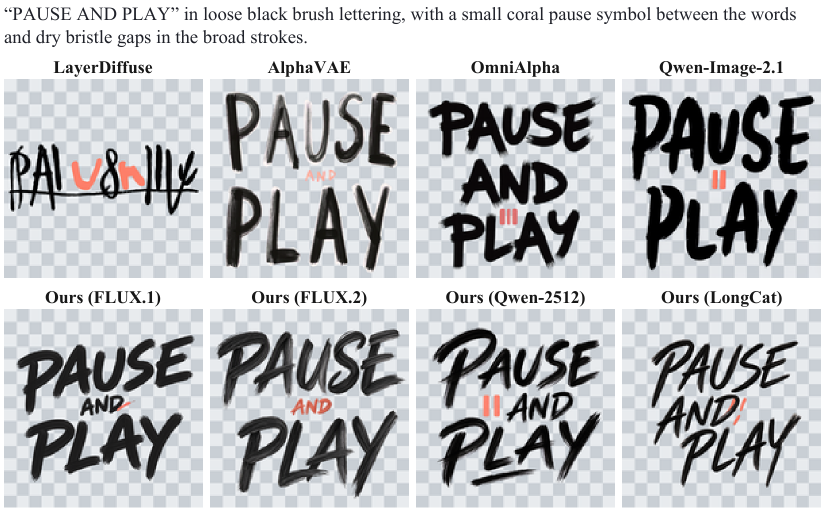}
\caption[]{Additional text-to-RGBA comparisons, set 1 of 2 (continued, cases 5--6 of 6).}
\end{figure}

\begin{figure}[p]
\centering
\setlength{\parskip}{6pt}
\captionsetup{skip=5pt}
\includegraphics[width=0.97\AlphaFigureWidth]{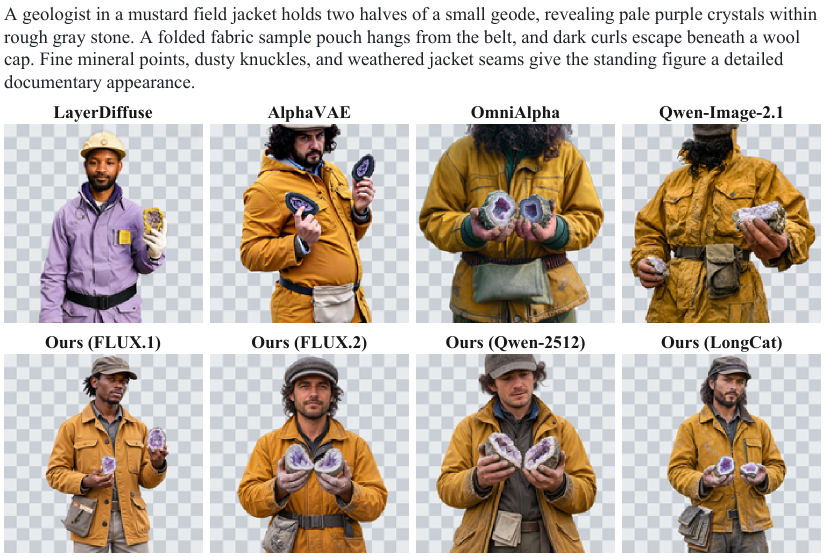}\par
\includegraphics[width=0.97\AlphaFigureWidth]{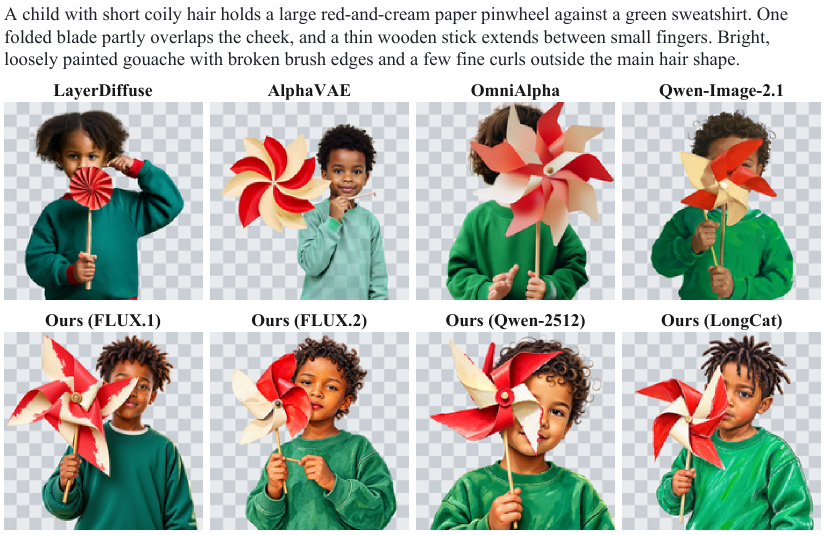}
\caption{Additional qualitative text-to-RGBA comparisons on 5 RGBA-Text prompts (set 2 of 2). Each case shows four baselines in the upper row and all four \method settings in the lower row. Complete RGBA canvases are retained.}
\label{fig:t2rgba-additional-2}
\end{figure}

\begin{figure}[p]
\ContinuedFloat
\centering
\setlength{\parskip}{6pt}
\captionsetup{skip=5pt}
\includegraphics[width=\AlphaFigureWidth]{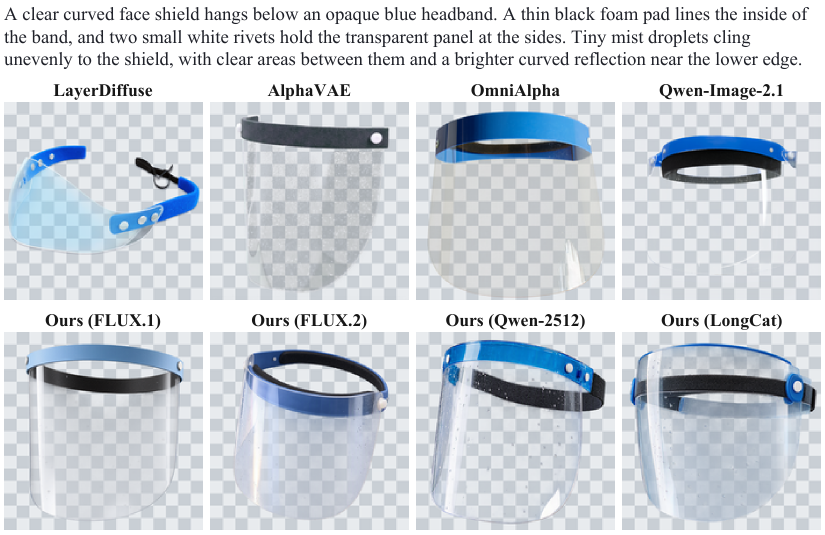}\par
\includegraphics[width=\AlphaFigureWidth]{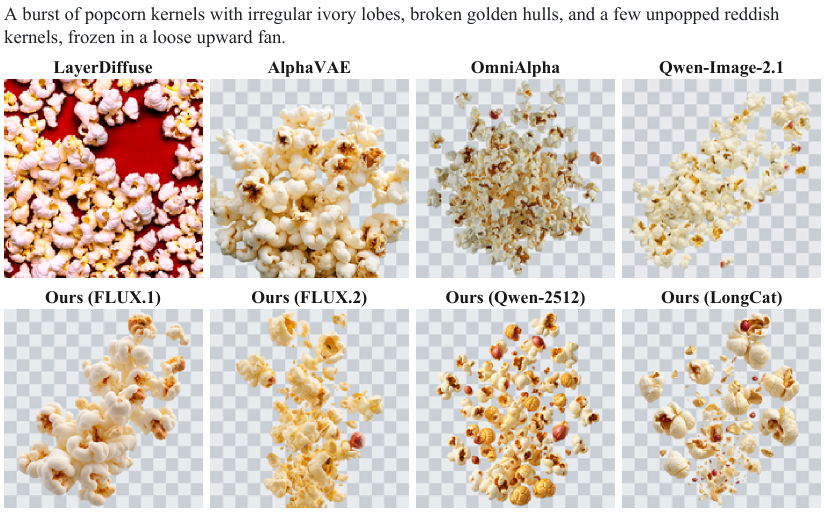}
\caption[]{Additional text-to-RGBA comparisons, set 2 of 2 (continued, cases 3--4 of 5).}
\end{figure}

\clearpage
\begin{figure}[!htbp]
\ContinuedFloat
\centering
\includegraphics[width=\AlphaFigureWidth]{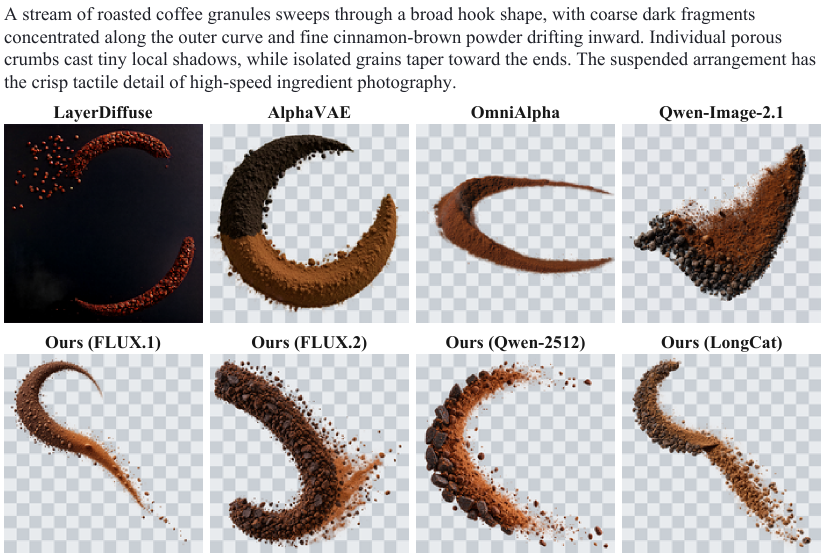}
\caption[]{Additional text-to-RGBA comparisons, set 2 of 2 (continued, case 5 of 5).}
\end{figure}

\FloatBarrier
\subsection{Image-to-RGBA Extraction}
\label{app:i2rgba-qualitative}

Figures~\ref{fig:i2rgba-additional} and~\ref{fig:i2rgba-additional-2} compare the four image-to-RGBA methods on fourteen additional RGBA-Image examples. Each target description appears above the complete input, target RGBA layer, and prediction canvases; the shared extraction prefix is stated in the figure caption. Where informative, one additional row examines either RGBA details or alpha boundaries using the same crop coordinates for every method. These visual comparisons complement the full-dataset scores in Table~\ref{tab:i2rgba-manual}.

\begin{figure}[p]
\centering
\setlength{\parskip}{6pt}
\captionsetup{skip=5pt}
\includegraphics[width=0.85\AlphaFigureWidth]{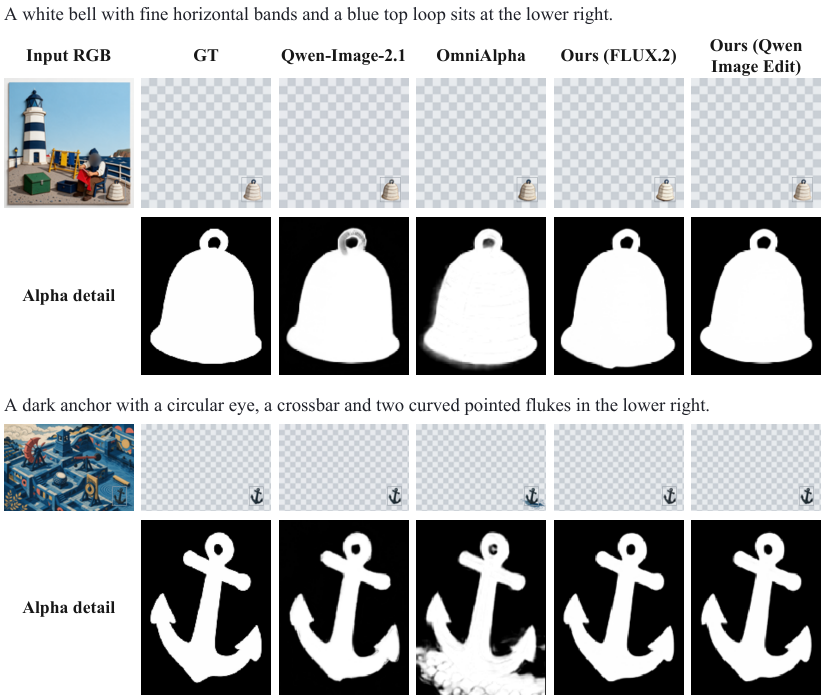}\par
\includegraphics[width=0.85\AlphaFigureWidth]{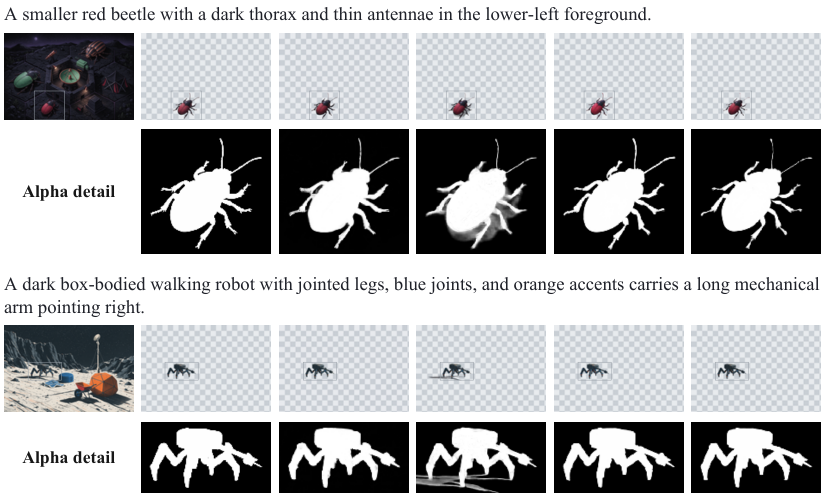}
\caption{Additional qualitative image-to-RGBA comparisons on seven RGBA-Image examples (set 1 of 2). Inputs, target layers, and all four methods retain their complete canvases. Informative alpha or RGBA details appear below the corresponding full views. Each displayed target description follows the shared prefix ``Extract the object described by the text, preserving fine details and transparency:''.}
\label{fig:i2rgba-additional}
\end{figure}

\begin{figure}[!htbp]
\ContinuedFloat
\centering
\setlength{\parskip}{6pt}
\captionsetup{skip=5pt}
\includegraphics[width=\AlphaFigureWidth]{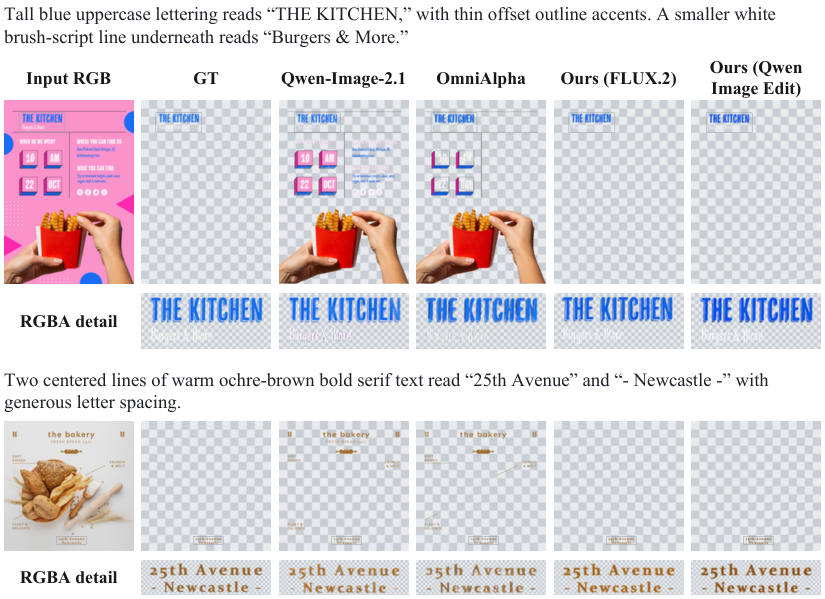}\par
\includegraphics[width=\AlphaFigureWidth]{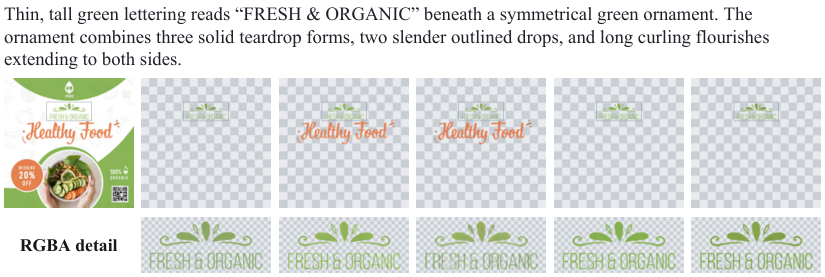}
\caption[]{Additional image-to-RGBA comparisons, set 1 of 2 (continued, cases 5--7 of 7).}
\end{figure}

\begin{figure}[!htbp]
\centering
\includegraphics[width=\AlphaFigureWidth]{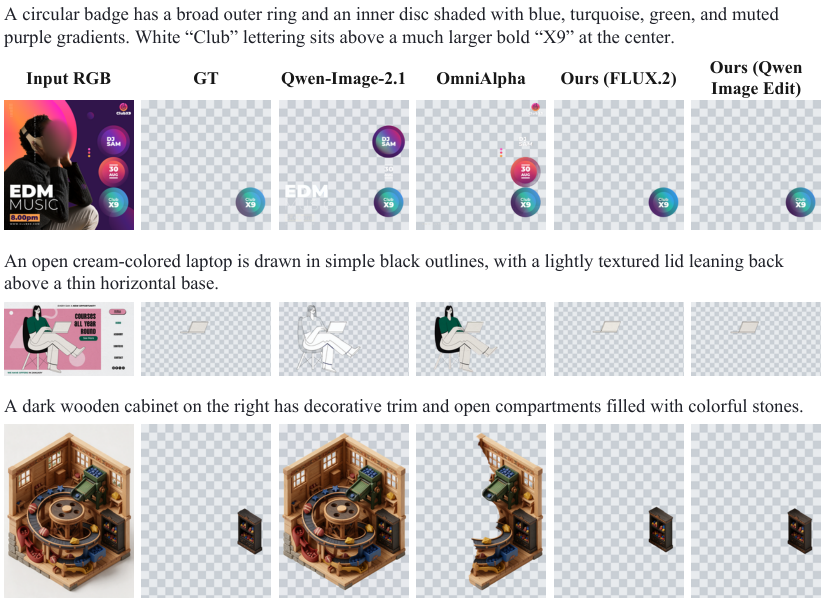}
\caption{Additional qualitative image-to-RGBA comparisons on seven RGBA-Image examples (set 2 of 2). Inputs, target layers, and all four methods retain their complete canvases. Informative alpha or RGBA details appear below the corresponding full views. Target descriptions use the shared extraction prefix in Figure~\ref{fig:i2rgba-additional}.}
\label{fig:i2rgba-additional-2}
\end{figure}

\begin{figure}[!htbp]
\ContinuedFloat
\centering
\setlength{\parskip}{6pt}
\captionsetup{skip=5pt}
\includegraphics[width=\AlphaFigureWidth]{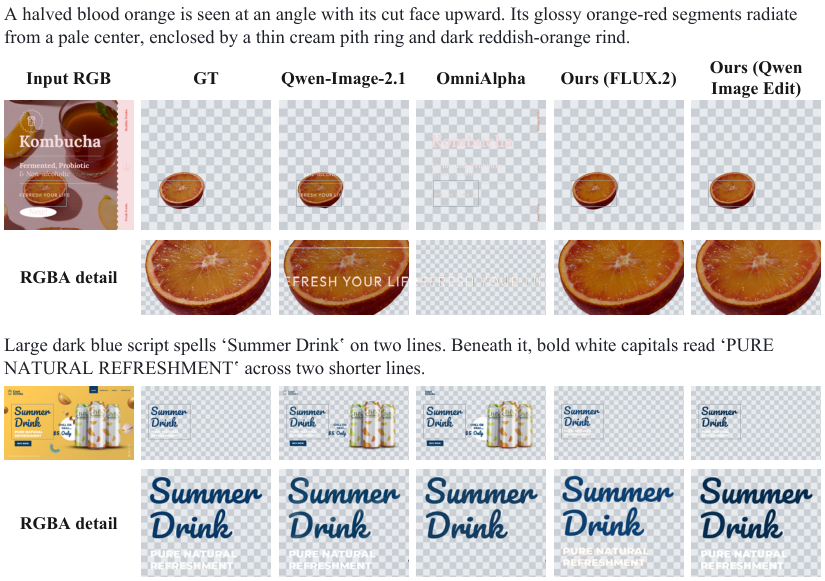}\par
\includegraphics[width=\AlphaFigureWidth]{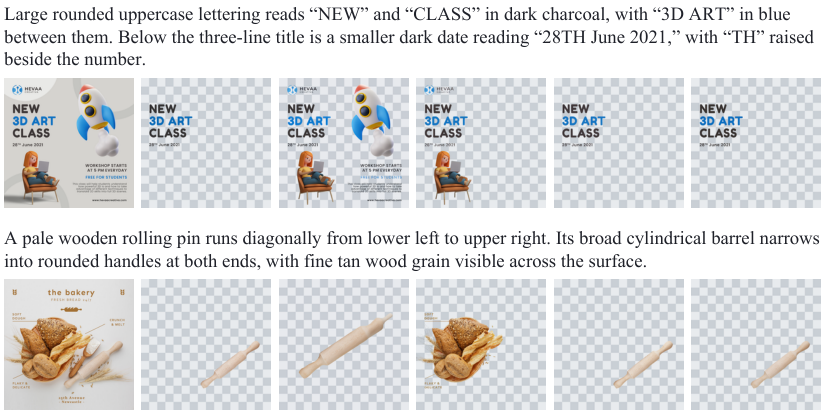}
\caption[]{Additional image-to-RGBA comparisons, set 2 of 2 (continued, cases 4--7 of 7).}
\end{figure}

\clearpage
\subsection{Unidirectional Feature Sharing}
\label{app:qualitative-coupled-vae}

Figure~\ref{fig:ablation-coupled-qualitative} compares \method with the coupled-decoder variant on RGBA-Image. Alpha-to-RGB feedback introduces visible foreground color shifts in these examples. The comparisons use the matched checkpoints and inference settings in Appendix~\ref{app:design-ablations}.

\begin{figure}[!htbp]
\centering
\setlength{\parskip}{6pt}
\captionsetup{skip=5pt}
\includegraphics[width=0.88\AlphaFigureWidth]{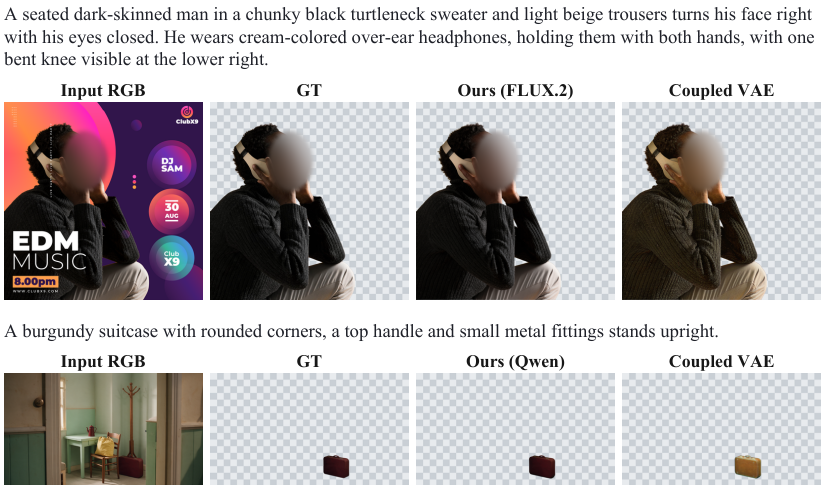}\par
\includegraphics[width=0.88\AlphaFigureWidth]{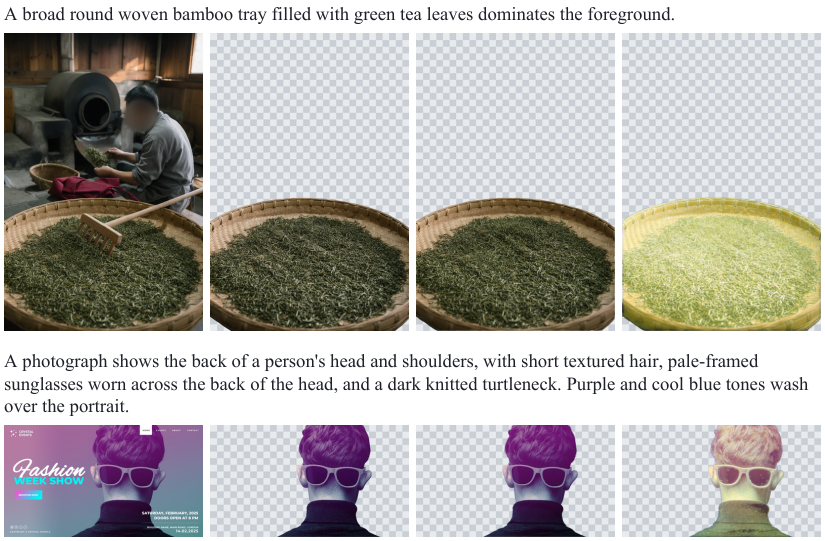}
\caption{Qualitative coupled-decoder ablation on FLUX.2 and Qwen-Image-Edit. \method preserves foreground colors more faithfully than the coupled decoder. Target descriptions use the shared extraction prefix in Figure~\ref{fig:i2rgba-additional}.}
\label{fig:ablation-coupled-qualitative}
\end{figure}

\clearpage
\subsection{Propagating Alpha Supervision to the DiT}
\label{app:qualitative-stopgrad}

Figure~\ref{fig:ablation-stopgrad-qualitative} compares \method with Stop-grad on Qwen-Image-Edit. The Stop-grad variant produces opacity artifacts around the price badge and wine bottle, omits the second line of lettering, and fills gaps between wheel spokes. The full model retains more accurate opacity structure in these cases.

\begin{figure}[!htbp]
\centering
\includegraphics[width=\AlphaFigureWidth]{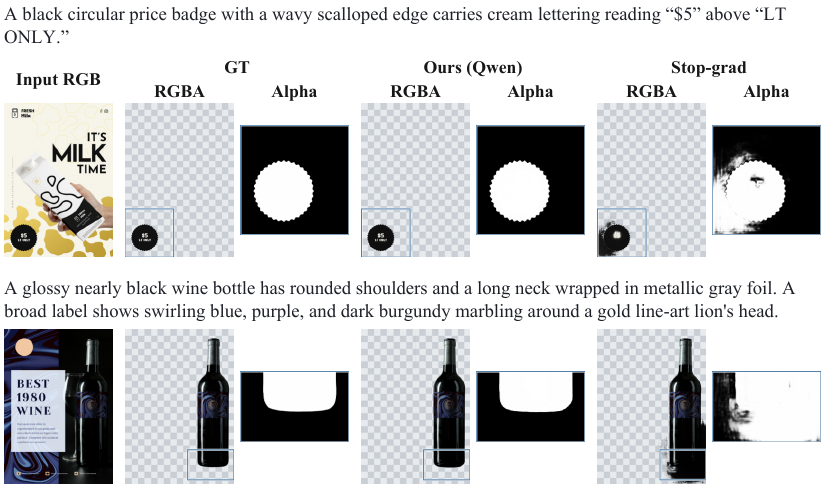}
\par
\includegraphics[width=\AlphaFigureWidth]{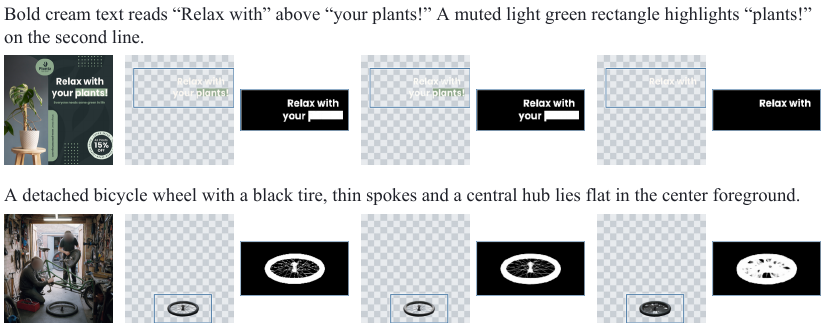}
\caption{Qualitative ablation of alpha supervision to the DiT on Qwen-Image-Edit. \method better preserves lettering and wheel spokes, with less stray opacity than Stop-grad. Target descriptions use the shared extraction prefix in Figure~\ref{fig:i2rgba-additional}.}
\label{fig:ablation-stopgrad-qualitative}
\end{figure}

\clearpage
\subsection{Endpoint-Aware Alpha Loss}
\label{app:qualitative-alpha-loss}

Figure~\ref{fig:ablation-alpha-loss-qualitative} compares EA Loss with standard alpha $L_1$ on Qwen-Image-Edit. The alpha-$L_1$ variant produces spurious opacity around the chocolate mold and craftsperson, and transparent holes inside the two opaque text panels. \method preserves the target opacity more closely in these examples.

\begin{figure}[!htbp]
\centering
\includegraphics[width=\AlphaFigureWidth]{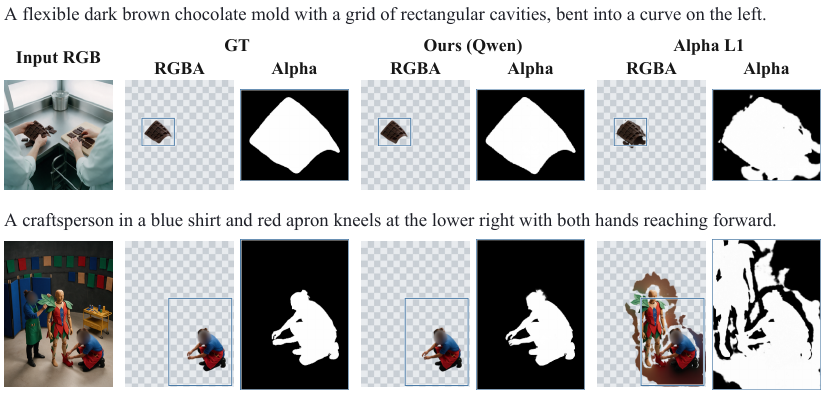}
\par
\includegraphics[width=\AlphaFigureWidth]{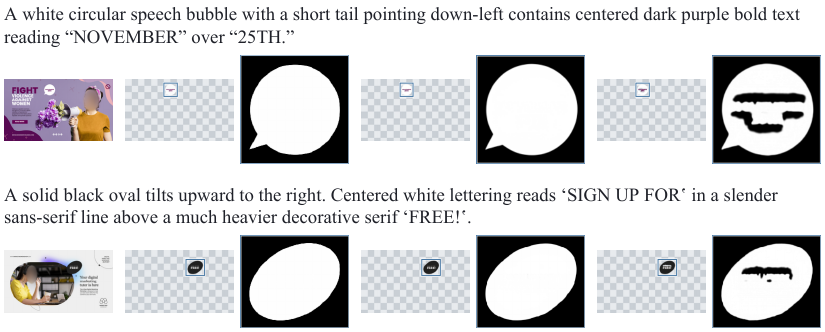}
\caption{Qualitative comparison of EA Loss and alpha $L_1$ on Qwen-Image-Edit. EA Loss reduces stray background opacity and transparent holes in opaque text regions. Target descriptions use the shared extraction prefix in Figure~\ref{fig:i2rgba-additional}.}
\label{fig:ablation-alpha-loss-qualitative}
\end{figure}

\clearpage
\subsection{Exploring RGBA Understanding Through Editing}
\label{app:rgba-understanding}

We view native RGBA image understanding as a potential advantage of jointly representing color and transparency in an RGBA VAE, while regarding high-quality RGBA generation as the more important capability. To explore this potential advantage, we conduct a qualitative editing experiment with Qwen-Image-2.1~\citep{qwen2026image21}. For each image layer, we compare the original RGBA input with the same layer composited over white and supplied with an all-one alpha channel. Under matched editing instructions and sampling settings, the paired outputs appear largely similar in overall appearance and opacity structure (Figure~\ref{fig:rgba-editing-input-comparison}). Future work could explore the distinctive value of RGBA understanding and the tasks in which it may offer practical benefits.

The six cases cover color modification, rotation, object removal, object addition, material modification, and style transfer. Each pair uses seed 42, 40 inference steps, a guidance scale of 1.0, and a $1024\times1024$ output canvas. We display the saved RGBA inputs and outputs on a common checkerboard.

\begin{figure}[!htbp]
\centering
\setlength{\parskip}{6pt}
\captionsetup{skip=5pt}
\includegraphics[width=\AlphaFigureWidth]{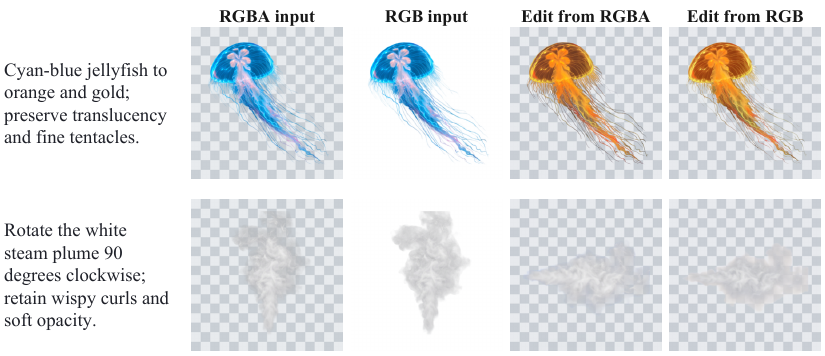}\par
\includegraphics[width=\AlphaFigureWidth]{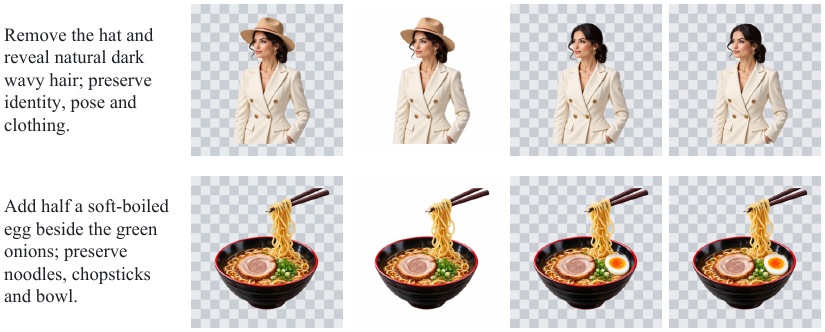}
\caption{Paired editing with Qwen-Image-2.1. The first two columns show the original RGBA input and its white-background composite with alpha set to one. The last two columns show the corresponding edited RGBA outputs. Each row uses the same prompt and sampling settings; the instruction is summarized on the left. RGBA images are displayed on a checkerboard. The two input conditions yield visually similar results in these examples.}
\label{fig:rgba-editing-input-comparison}
\end{figure}

\begin{figure}[!htbp]
\ContinuedFloat
\centering
\includegraphics[width=\AlphaFigureWidth]{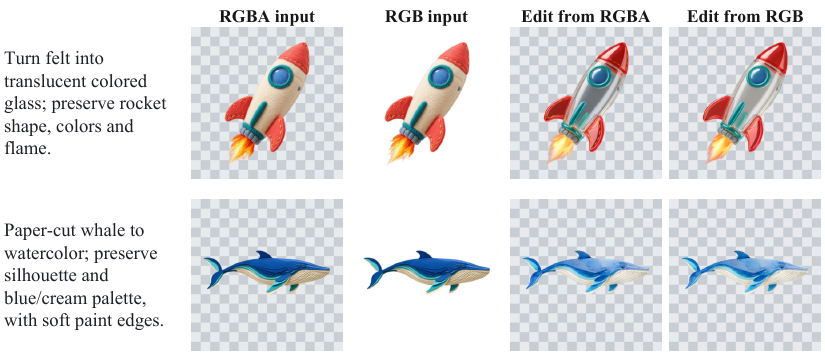}
\caption[]{Paired editing with Qwen-Image-2.1 (continued, cases 5--6 of 6).}
\end{figure}

\end{document}